\pdfoutput=1
\documentclass[twoside]{article}

\usepackage[accepted]{aistats2019}

\usepackage[numbers]{natbib}

\PassOptionsToPackage{square,sort,comma,numbers}{natbib}

\usepackage[colorlinks=true,allcolors=black!40!blue]{hyperref}
\usepackage{mycustom}

\begin{document}

% If your paper is accepted and the title of your paper is very long,
% the style will print as headings an error message. Use the following
% command to supply a shorter title of your paper so that it can be
% used as headings.
%
\runningtitle{Buffered Stochastic Variational Inference}

% If your paper is accepted and the number of authors is large, the
% style will print as headings an error message. Use the following
% command to supply a shorter version of the authors names so that
% they can be used as headings (for example, use only the surnames)
%
\runningauthor{Rui Shu, Hung H. Bui, Jay Whang, Stefano Ermon}

\twocolumn[

\aistatstitle{Training Variational Autoencoders with\\
Buffered Stochastic Variational Inference}

\aistatsauthor{
Rui Shu\\
Stanford University
\And
Hung H. Bui\\
VinAI
\And
Jay Whang\\
Stanford University
\And 
Stefano Ermon\\
Stanford University}
\aistatsaddress{} 
]

% \aistatsaddress{} 
% ]

\begin{abstract}
%\rs{I use the terminology amortized inference model in the paper. Should I switch it for recognition network?}

%\hb{no don't worry, let me take a crack at this and you can look at it later}
The recognition network in deep latent variable models such as variational autoencoders (VAEs) relies on amortized inference for efficient posterior approximation that can scale up to large datasets. However, this technique has also been demonstrated to select suboptimal variational parameters, often resulting in considerable additional error called the {\em amortization gap}. 
%\rs{this sentence feels too long.} 
To close the amortization gap and improve the training of the generative model, recent works have introduced an additional refinement step that applies stochastic variational inference (SVI) to improve upon the variational parameters returned by the amortized inference model. In this paper, we propose the \emph{Buffered Stochastic Variational Inference} (BSVI), a new refinement procedure that makes use of SVI's sequence
%\jw{``sequence of intermediate variational proposal distributions''?}
of intermediate variational proposal distributions and their corresponding importance weights to construct a new generalized importance-weighted lower bound.
%\hb{I wonder if SVI is good name to refer to the previous work. Maybe SVI-Amortization}
%\rs{I've added a patch to Section 3 indicating that we always assume amortized initialization. But I agree we should revisit this and find a propose solution}
% \rs{I'd prefer the abstract not to mention things like computation cost. I consider it a cherry-on-top, not a critical detail.}
%yielding an improved refinement procedure with the same computational cost. 
We demonstrate empirically that training the variational autoencoders with BSVI consistently out-performs SVI, yielding an improved training procedure for VAEs.
\end{abstract}

\section{Introduction} % \rs{I refuse to use full-caps. Bite me.}

Deep generative latent-variable models are important building blocks in current approaches to a host of challenging high-dimensional problems including density estimation \cite{kingma2013auto,sonderby2016ladder,shu2016bottleneck}, semi-supervised learning \cite{kingma2014semi,kuleshov2017deep} and 
%\jw{``representation learning for downstream...'' for parallelism?} 
representation learning for downstream tasks \cite{chen2018isolating,watter2015embed,banijamali2017robust,li2017infogail}. To train these models, the principle of maximum likelihood is often employed. However, maximum likelihood is often intractable due to the difficulty of marginalizing the latent variables. Variational Bayes addresses this by instead providing a tractable lower bound of the log-likelihood, which serves as a surrogate target for maximization. Variational Bayes, however, introduces a \emph{per sample} optimization subroutine to find the variational proposal distribution that best matches the true posterior distribution (of the latent variable given an input observation). To amortize the cost of this optimization subroutine, the variational autoencoder introduces an amortized inference model that learns to predict the best proposal distribution given an input observation \cite{kingma2013auto,rezende2014stochastic,gershman2014amortized,zhao2018lagrangian}.

%\rs{Things to cite: \cite{kingma2014semi,kingma2016improved,chen2018isolating,higgins2016beta,babaeizadeh2017stochastic,liu2017unsupervised,watter2015embed,banijamali2017robust,wu2016quantitative,rainforth2018tighter}}

%\hb{State the inference suboptimality issue. State papers that try to fix via SVI. State that SVI wastes intermediate distributions. }
Although the computational efficiency of amortized inference has enabled latent variable models to be trained at scale on large datasets \cite{pu2016variational,gulrajani2016pixelvae}, amortization introduces an additional source of error in the approximation of the posterior distributions if the amortized inference model fails to predict the optimal proposal distribution. This additional source of error, referred to as the amortization gap \cite{cremer2018inference}, causes variational autoencoder training to further deviate from maximum likelihood training \cite{cremer2018inference,shu2018amortized}.

To improve
%better approximate maximum likelihood 
training, numerous methods have been developed to reduce the amortization gap. In this paper, we focus on a class of methods \cite{krishnan2017challenges,kim2018semi,marino2018iterative} that takes an initial proposal distribution predicted by the amortized inference model and refines this initial distribution with the application of \emph{Stochastic Variational Inference} (SVI) \cite{hoffman2013stochastic}. Since SVI applies gradient ascent to iteratively update the proposal distribution, a by-product of this procedure is a trajectory of proposal distributions $(q_0, \ldots, q_k)$ and their corresponding importance weights $(w_0, \ldots w_k)$. The intermediate distributions are discarded, and only the last distribution $q_k$ is retained for updating the generative model. Our key insight is that the intermediate importance weights can be \emph{repurposed} to further improve training. Our contributions are as follows
\begin{enumerate}
    \item We propose a new method, \emph{Buffered Stochastic Variational Inference} (BSVI), that takes advantage of the intermediate importance weights and constructs a new lower bound (the BSVI bound).
    \item We show that the BSVI bound is a special instance of a family of generalized importance-weighted lower bounds.
    \item We show that training variational autoencoders with BSVI consistently outperforms SVI, demonstrating the effectiveness of leveraging the intermediate weights.
\end{enumerate}
Our paper shows that BSVI is an attractive replacement of SVI with minimal development and computational overhead. 

\section{Background and Notation}

We consider a latent-variable generative model $p_\theta(x, z)$  where $x \in \X$ is observed, $z \in \Z$ is latent, and $\theta$ are the model's parameters. The marginal likelihood $p_\theta(x)$ is intractable but can be lower bounded by the evidence lower bound (ELBO)
\begin{align}
\ln p_\theta(x) \ge \Expect_{q(z)} \brac{\ln \frac{p_\theta(x, z)}{q(z)}} = \Expect_{q(z)} \ln w(z), \label{eq:elbo}
\end{align}
which holds for any distribution $q(z)$. Since the gap of this bound is exactly the Kullback-Leibler divergence $\KL{q(z)}{p_\theta(z \giv x)}$, $q(z)$ is thus the variational approximation of the posterior.
%\mh{We should explicitly state that $\KL{q}{p}$ denotes KL divergence.} 
Furthermore, by viewing $q$ as a proposal distribution in an importance sampler, we refer to $w(z)=\frac{p_\theta(x, z)}{q(z)}$ as an unnormalized importance weight. Since $w(z)$ is a random variable, the variance can be reduced by averaging the importance weights derived from i.i.d samples from $q(z)$. This yields the Importance-Weighted Autonenocder (IWAE) bound~\cite{burda2015importance},
\begin{align}
\ln p_\theta(x) \ge \Expect_{z_1\ldots z_k \stackrel{\text{i.i.d.}}{\sim} q} \brac{\ln \frac{1}{k} \sum_{i=1}^k w(z_i)} \ge \ELBO, \label{eq:iwae}
\end{align}
which admits a tighter lower bound than the ELBO~\cite{burda2015importance,domke2018importance}.

\subsection{Stochastic Variational Inference}

The generative model can be trained by jointly optimizing $q$ and $\theta$ to maximize the lower bound over the data distribution $\hat{p}(x)$. Supposing the variational family $\Q = \set{q(z \scolon \lambda)}_{\lambda \in \Lambda}$ is parametric and indexed by the parameter space $\Lambda$ (e.g. a Gaussian variational family indexed by mean and covariance parameters),
%\s{define domain $\Lambda$, used later} \s{maybe give an example, to keep things concrete, mean and variance of a gaussian}
%\mh{Pedantic notation nit: I prefer $q_\lambda(z)$ or $q(z; \lambda)$ to $q(z\mid \lambda)$, since $\lambda$ isn't a random variable.}
%\rs{Noted}
the optimization problem becomes
\begin{align}\label{eq:svi_intro}
    \max_{\theta} \Expect_{\hat{p}(x)} \brac{ \max_{\lambda} \Expect_{q(z \scolon \lambda)} \ln w(z \scolon \lambda, \theta) }.
\end{align}
where importance weight $w$ is now
\begin{align}
    w(z \scolon \lambda, \theta) = \frac{p_\theta(x, z)}{q(z \scolon \lambda)}.
\end{align}
For notational simplicity, we omit the dependency on $x$. For a fixed choice of $\theta$ and $x$, \cite{krishnan2017challenges} proposed to optimize $\lambda$ via gradient ascent, where one initializes with $\lambda_0$ and takes successive steps of
\begin{align}
    \lambda_{i + 1} \gets \lambda_i + \eta \nabla_{\lambda_i} \ELBO,
\end{align}
%\mh{Can we get rid of the $\cdot$? It looks like a dot product to me.}\rs{Done :)}
for which the ELBO gradient with respect to $\lambda_i$ can be approximated via Monte Carlo sampling as
\begin{align}\label{eq:svi_grad}
    \nabla_{\lambda_i} \ELBO \approx \frac{1}{m} \sum_{j=1}^m \nabla_{\lambda_i} \ln w(z_{\lambda_i}(\eps_i^{(j)}) \scolon \lambda_i, \theta)
\end{align}
where $z_i^{(j)} = z_{\lambda_i}(\eps_i^{(j)}) \sim q(z \scolon \lambda_i)$ is reparameterized as a function of $\lambda_i$ and a base distribution $p_0(\epsilon)$. We note that $k$ applications gradient ascent generates a trajectory of variational parameters $(\lambda_0, \ldots, \lambda_k)$, where we use the final parameter $\lambda_k$ for the approximation. Following the convention in \cite{hoffman2013stochastic}, we refer to this procedure as \emph{Stochastic Variational Inference} (SVI).

\subsection{Amortized Inference Suboptimality}
The SVI procedure introduces an inference \emph{subroutine} that optimizes the proposal distribution $q(z \scolon \lambda)$ \emph{per sample}, which is computationally costly. \cite{kingma2013auto,rezende2014stochastic} observed that the computational cost of inference can be \emph{amortized} by introducing an inference model $f_\phi: \X \to \Lambda$, parameterized by $\phi$, that directly seeks to learn the mapping $x \mapsto \lambda^*$ from each sample $x$ to an optimal $\lambda^*$ that solves the maximization problem
\begin{align}
    \lambda^* = \argmax_{\lambda} \Expect_{q(z \scolon \lambda)} \ln \frac{p_\theta(x, z)}{q(z \scolon \lambda)}.
\end{align}
This yields the amortized ELBO optimization problem
\begin{align}\label{eq:aelbo}
    \max_{\theta, \phi} \Expect_{\hat{p}(x)} \brac{ \Expect_{q(z \scolon f_\phi(x))} \ln \frac{p_\theta(x, z)}{q(z \scolon f_\phi(x))}},
\end{align}
where $q(z \scolon f_\phi(x))$ can be concisely rewritten (with a slight abuse of notation) as $q_\phi(z \giv x)$ to yield the standard variational autoencoder objective \cite{kingma2013auto}.

While computationally efficient, the influence of the amortized inference model on the training dynamics of the generative model has recently come under scrutiny \cite{cremer2018inference,krishnan2017challenges,kim2018semi,shu2018amortized}. A notable consequence of amortization is the \emph{amortization gap}
\begin{align}
    \KL{q_\phi(z \giv x)}{p_\theta(z \giv x)} - \KL{q(z \scolon \lambda^*)}{p_\theta(z \giv x)}
\end{align}
which measures the additional error incurred when the amortized inference model is used instead of the optimal $\lambda^*$ for approximating the posterior \cite{cremer2018inference}.
A large amortization gap can present a potential source of concern since it  introduces further deviation from the maximum likelihood objective \cite{shu2018amortized}.
% \s{might want to define these components a bit more precisely/formally, or skip the formulas since i don't think you'll ever use them again}

\subsection{Amortization-SVI Hybrids}

To close the amortization gap, \cite{krishnan2017challenges} proposed to blend amortized inference with SVI. Since SVI requires one to initialize $\lambda_0$, a natural solution is to set $\lambda_0 = f_\phi(x)$. Thus, SVI is allowed to fine-tune the initial proposal distribution found by the amortized inference model and reduce the amortization gap. Rather than optimizing $\theta, \phi$ jointly with the amortized ELBO objective \cref{eq:aelbo}, the training of the inference and generative models is now decoupled; $\phi$ is trained to optimize the amortized ELBO objective, but $\theta$ is trained to approximately optimize \cref{eq:svi_intro}, where $\lambda^* \approx \lambda_k$ is approximated via SVI. To enable end-to-end training of the inference and generative models, \cite{kim2018semi} proposed to backpropagate through the SVI steps via a finite-difference estimation of the necessary Hessian-vector products. Alternatively, \cite{marino2018iterative} adopts a learning-to-learn framework where an inference model iteratively outputs $\lambda_{i+1}$ as a function of $\lambda_i$ and the ELBO gradient.

%\s{later in the discussion, we should probably reiterate how buffer is related/can be combined with amortization}\rs{agreed. Though I haven't decided on where to do it yet. Any thoughts?}

\section{Buffered Stochastic Variational Inference}
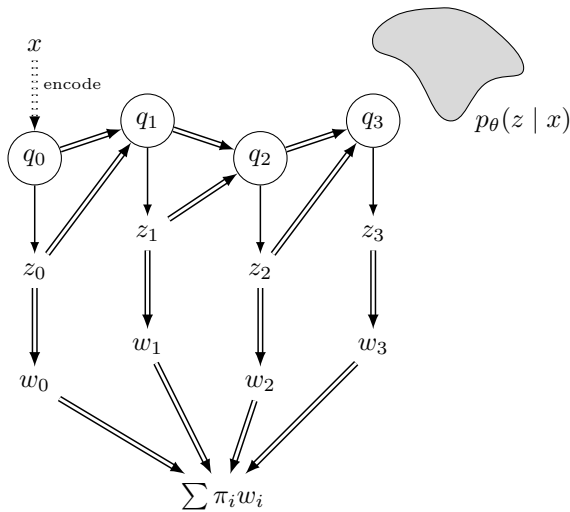
\begin{figure}[!h]
\centering
\begin{tikzpicture}
\node[none] (x) {$x$};
\node[none, yshift=-.5cm, xshift=+.5cm] () {{\tiny encode}};

\node[circ, yshift=-1.5cm] (q0) {$q_0$};
\node[none, yshift=-3cm] (z0) {$z_0$};
\node[none, yshift=-4.5cm] (w0) {$w_0$};

\node[circ, xshift=+1.5cm, yshift=-1cm] (q1) {$q_1$};
\node[none, xshift=+1.5cm, yshift=-2.5cm] (z1) {$z_1$};
\node[none, xshift=+1.5cm, yshift=-4cm] (w1) {$w_1$};

\node[circ, xshift=+3cm, yshift=-1.5cm] (q2) {$q_2$};
\node[none, xshift=+3cm, yshift=-3cm] (z2) {$z_2$};
\node[none, xshift=+3cm, yshift=-4.5cm] (w2) {$w_2$};

\node[circ, xshift=+4.5cm, yshift=-1cm] (q3) {$q_3$};
\node[none, xshift=+4.5cm, yshift=-2.5cm] (z3) {$z_3$};
\node[none, xshift=+4.5cm, yshift=-4cm] (w3) {$w_3$};

\node[none, xshift=+6.5cm, yshift=-1cm] (p) {$p_\theta(z \giv x)$};
\node[none, xshift=+2.5cm, yshift=-6cm] (bsvi) {$\sum \pi_i w_i$};
\draw
(x) edge [dotdoubleconnect] (q0)

(q0) edge [doubleconnect] (q1)
(q1) edge [doubleconnect] (q2)
(q2) edge [doubleconnect] (q3)

(q0) edge [connect] (z0)
(z0) edge [doubleconnect] (w0)
(z0) edge [doubleconnect] (q1)
(w0) edge [doubleconnect] (bsvi)

(q1) edge [connect] (z1)
(z1) edge [doubleconnect] (w1)
(z1) edge [doubleconnect] (q2)
(w1) edge [doubleconnect] (bsvi)

(q2) edge [connect] (z2)
(z2) edge [doubleconnect] (w2)
(z2) edge [doubleconnect] (q3)
(w2) edge [doubleconnect] (bsvi)

(q3) edge [connect] (z3)
(z3) edge [doubleconnect] (w3)
(w3) edge [doubleconnect] (bsvi)
;
\draw[fill=gray!30] plot[smooth, tension=.8] coordinates {(5.5, 0.5) (6, 0.3) (6.5, -0.3) (6, -0.5) (5.5, -1) (5, -0.3) (4.5, 0) (4.8, 0.4) (5.5, 0.5)};
\end{tikzpicture}
\caption{Idealized visualization of Buffered Stochastic Variational Inference. Double arrows indicate deterministic links, and single arrows indicate stochastic links that involve sampling. The dotted arrow from $x$ to $q_0$ denotes that the initial variational parameters are given by the encoder. For notational simplicity, we omitted the dependence of $q_{1:k}$ on $x$ and the model parameters $\phi, \theta$. 
} 
%\mh{I like this figure. It seems like there should be arrows from x to $q_1, q_2$, etc. That deemphasizes $q_0$'s special dependence on x, though---maybe you could add a $\phi$ pointing to $q_0$ to denote the inference network? Also, should there be an arrow from $w_i$ to $q_{i+1}$?}
%\rs{Yes, we should add $\phi$ and $x$ (and even $\theta$) for completeness, and have $x, \theta$ point to all the $q$'s. Regarding arrow from $w$ to $q$, I wonder if it's technically needed. Let's make a note of this, but try to keep the graph as simple as possible.}
\label{fig:bsvi}
\end{figure}

In this paper, we focus on the simpler, decoupled training procedure described by \cite{krishnan2017challenges} and identify a new way of improving the SVI training procedure (orthogonal to the end-to-end approaches in \cite{kim2018semi,marino2018iterative}). Our key observation is that, as part of the gradient ascent estimation in \cref{eq:svi_grad}, the SVI procedure necessarily generates a sequence of importance weights $(w_0, \ldots, w_k)$, where $w_i = w(z_i \scolon \lambda_i, \theta)$. 
%\s{this is a key observation and needs more text discussion. particularly, need to tie with the notation you used in 2.1 (in terms of lambdas). refer to eq. 4, match subscript notation (t vs k). in 2.1, you never mention when/where the importance weights are pre-computed. if there's space (or in appendix), could be helpful to provide and refer to SVI pseudocode}
Since $(\ln w_k)$ likely achieves the highest ELBO, the intermediate weights $(w_0, \ldots, w_{k-1})$ are subsequently discarded in the SVI training procedure, and only $\nabla_\theta \ln w_k$ is retained for updating the generative model parameters. However, if the preceding proposal distributions $(q_{k-1}, q_{k-2}, \ldots)$ are also reasonable approximations of the posterior, then it is potentially wasteful to discard their corresponding importance weights.
%\s{maybe give some intuition here: $w_{k-1}$ is likely almost as good as $w_k$, so why not use it?}
A natural question to ask then is whether the full trajectory of weights $(w_0, \ldots, w_k)$ can be leveraged to further improve the training of the generative model. 

%\s{not sure if there's time, but it might be interesting to cook up some synthetic 1D mixture of 2 gaussian case, where we can show the various proposals (we start covering one mode, then we shift to the other) and it becomes clear you want to use the buffer. a figure like that would be "worth" more than a lot of math for many readers}
%\rs{Noted! I'll make one later (once everything else is written). If I make it before the deadline, I'll add it in!}

%\s{never say your stuff is simple :)}
%\mh{Readers like simple! But sometimes reviewers don't :(. Maybe "easy-to-implement" is a reasonable compromise?} 
Taking inspiration from IWAE's weight-averaging mechanism, we propose a modification to the SVI procedure where we simply keep a \emph{buffer} of the entire importance weight trajectory and use an average of the importance weights $\sum_i \pi_i w_i$ as the objective in training the generative model.\footnote{For simplicity, we use the uniform-weighting $\pi_i = 1/(k+1)$ in our base implementation of BSVI. In \cref{sec:hyperweight}, we discuss how to optimize $\pi$ during training.} The generative model is then updated with the gradient $\nabla_\theta \ln \sum_i \pi_i w_i$. We call this procedure \emph{Buffered Stochastic Variational Inference} (BSVI) and denote $\ln \sum_i \pi_i w_i$ as the BSVI objective. We describe the BSVI training procedure in \cref{alg:bsvi} and contrast it with SVI training. For notational simplicity, we shall always imply initialization with an amortized inference model when referring to SVI and BSVI.
%\s{agreed, this paragraph could be confusing}

\begin{algorithm}[!h]
	\caption{\small\textbf{Training with Buffered Stochastic Variational Inference}. We contrast training with SVI versus BSVI. We denote the stop-gradient operation with $\lceil\cdot\rceil$, reflecting that we do not backpropagate through the SVI steps.}
	\label{alg:bsvi}
	\begin{algorithmic}[1]
        \State \textbf{Inputs}: $\D = \set{x^{(1)}, \ldots, x^{(n)}}$.
        \For{$t = 1\ldots T$}
            \State $x \sim \D$
            \State $\lambda_0 \gets f_{\phi_t}(x)$
            %\State $\hat{\theta} \gets \lceil \theta \rceil$
                \For{$i = 0\ldots k$}
                    \State $z_i \sim q(z \scolon \lambda_i)$ \Comment{reparameterize as $z_{\lambda_i}(\eps)$}
                    \State $w(z \scolon \lambda_i, \theta) \gets {p_{{\theta}}(x, z_i)}/{q(z_i \scolon \lambda_i)}$
                    \If{$i < k$}
                        \State $\lambda_{i + 1} \gets \lceil\lambda_i + \eta \nabla_{\lambda_i} \ln w(z \scolon \lambda_i, \theta)\rceil$
                        %\State $\lambda_{i + 1} \gets \lambda_i + \eta \nabla_{\lambda_i} \ln w_i$
                    \EndIf
                \EndFor
            \State $\phi_{t+1} \gets \phi_t + \nabla_{\phi_t} \ln w(z_0 \scolon \lambda_0, \theta_t)$
            \If{Train with SVI}
                \State $\theta_{t+1} \gets \theta_t + \nabla_{\theta_t} \ln w(z_k \scolon \lambda_k, \theta_t)$
            \ElsIf{Train with BSVI}
                \State $\theta_{t+1} \gets \theta_t + \nabla_{\theta_t} \ln \sum_i \pi_i w(z_i \scolon \lambda_i, \theta_t)$
            \EndIf
        \EndFor
	\end{algorithmic}
\end{algorithm}

\begin{figure}[!h]

\begin{subfigure}[b]{\columnwidth}
\centering
\begin{tikzpicture}
\node[circ, xshift=+0cm] (l0) {$\lambda_0$};
\node[circ, xshift=+2cm] (l1) {$\lambda_1$};
\node[circ, xshift=+4cm] (l2) {$\lambda_2$};
\node[circ, xshift=+6cm] (l3) {$\lambda_3$};
\node[circ, xshift=+0cm, yshift=-2cm] (z0) {$z_0$};
\node[circ, xshift=+2cm, yshift=-2cm] (z1) {$z_1$};
\node[circ, xshift=+4cm, yshift=-2cm] (z2) {$z_2$};
\node[circ, xshift=+6cm, yshift=-2cm] (z3) {$z_3$};
\draw
(l0) edge [doubleconnect] (l1)
(l0) edge [connect] (z0)
(z0) edge [doubleconnect] (l1)
(l1) edge [doubleconnect] (l2)
(l1) edge [connect] (z1)
(z1) edge [doubleconnect] (l2)
(l2) edge [connect] (z2)
(l2) edge [doubleconnect] (l3)
(z2) edge [doubleconnect] (l3)
(l3) edge [connect] (z3)
;
\end{tikzpicture}
\caption{Dependent proposal distributions}
\label{fig:dependent_proposals}
\end{subfigure}

\begin{subfigure}[b]{\columnwidth}
\centering
\begin{tikzpicture}
\node[circ, xshift=+0cm] (z0) {$z_0$};
\node[circ, xshift=+2cm] (z1) {$z_1$};
\node[circ, xshift=+4cm] (z2) {$z_2$};
\node[circ, xshift=+6cm] (z3) {$z_3$};
\draw
(z0) edge [lconnect, bend left=45] (z1)
(z0) edge [lconnect, bend left=45] (z2)
(z0) edge [lconnect, bend left=45] (z3)
(z1) edge [lconnect, bend left=45] (z2)
(z1) edge [lconnect, bend left=45] (z3)
(z2) edge [lconnect, bend left=45] (z3)
;
\end{tikzpicture}
\caption{Dependent samples}
\label{fig:dependent_zs}
\end{subfigure}%

\caption{Graphical model for dependent proposal distributions and samples. When $\lambda_{1:k}$ is marginalized, the result is a joint distribution of dependent samples. For notational simplicity, the dependency on $\theta$ is omitted.}
\end{figure}
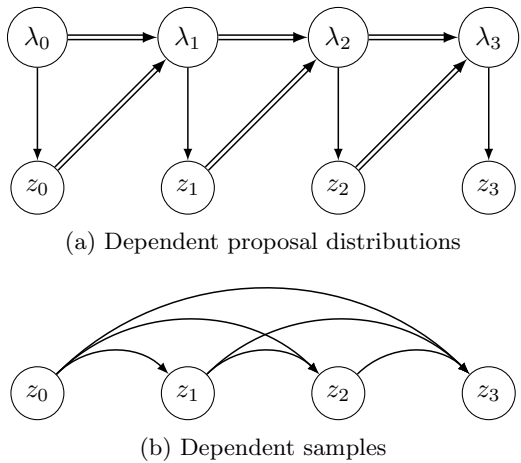

\section{Theoretical Analysis}

An important consideration is whether the BSVI objective serves as a valid lower bound to the log-likelihood $\ln p_\theta(x)$. A critical challenge in the analysis of the BSVI objective is that the trajectory of variational parameters $(\lambda_0, \ldots, \lambda_k)$ is actually a sequence of \emph{statistically-dependent random variables}. This statistical dependency is a consequence of SVI's stochastic gradient approximation in \cref{eq:svi_grad}. We capture this dependency structure in \Cref{fig:dependent_proposals}, which shows that each $\lambda_{i+1}$ is only deterministically generated after $z_i$ is sampled. When the proposal distribution parameters $\lambda_{0:k}$ are marginalized, the resulting graphical model is a joint distribution over $q(z_{0:k} \giv x)$. To reason about such a joint distribution, we introduce the following generalization of the IWAE bound.

\begin{restatable}{theorem}{ThmGIWAE}\label{thm:giwae}
%\s{try making statement as self contained as possible. e.g., introduce p}
Let $p(x,z)$ be a distribution where $z\in \mathcal{Z}$. Consider a joint proposal distribution $q(z_{0:k})$ over $\mathcal{Z}^k$. 
Let $v(i)\subset \set{0, \ldots, k} \setminus \set{i}$ for all $i$, and $\pi$ be a categorical distribution over $\set{0, \ldots, k}$. The following construction, which we denote the Generalized IWAE Bound, is a valid lower bound of the log-marginal-likelihood
%\s{missing quantification over $\pi$. should be for all pi}
%\s{add summation upper / lower bounds}
\begin{align}
    \Expect_{q(z_{0:k})} \ln \sum_{i=0}^k \pi_i \frac{p(x, z_i)}{q(z_i \giv z_{v(i)})} \le \ln p(x),
\end{align}
\end{restatable}
%\s{not entirely sure why the c is needed}
%\hb{agreed, we should get rid of c in the theorem. the current proof doesn't mention c}

The proof follows directly from the linearity of expectation when using $q(z_{0:k})$ for importance-sampling to construct an unbiased estimate of $p_\theta(x)$, followed by application of Jensen's inequality. A detailed proof is provided in \cref{app:proofs}. 
%\mh{I wonder if this is a corollary of one of the nested Monte Carlo/FIVO/Variational SMC arguments? Rainforth et al.'s seems like maybe the most general.} 
%\rs{That would be quite problematic. I checked Nested MC/FIVO/VSMC just now and searched for ``dependent''---nothing similar came up. I'll do a closer search later...}

Notably, if $q(z_{0:k}) = \prod_i q(z_i)$, then \cref{thm:giwae} reduces to the IWAE bound. \cref{thm:giwae} thus provides a generalization of IWAE, where the samples drawn are potentially \emph{non-independently and non-identically} distributed. 
%\s{explicitly discuss how to cast iwae in this framework: IWAE corresponds to xxx choosing yyy}
\Cref{thm:giwae} thus provides a way to construct new lower bounds on the log-likelihood whenever one has access to a set of non-independent samples. 

In this paper, we focus on a special instance where a chain of samples is constructed from the SVI trajectory. We note that the BSVI objective can be expressed as
\begin{align}\label{eq:bsvi}
    \Expect_{q(z_{0:k} \giv x)} \ln \sum_{i=0}^k \pi_i w_i = 
    \Expect_{q(z_{0:k} \giv x)} \ln \sum_{i=0}^k \pi_i \frac{p_\theta(x, z_i)}{q(z_i \giv z_{<i}, x)}.
\end{align}
Note that since $\lambda_i$ can be deterministically computed given $(x, z_{<i})$, it is therefore admissible to interchange the distributions $q(z_i \giv z_{<i}, x) = q(z_i \giv \lambda_i)$. The BSVI objective is thus a special case of the Generalized IWAE bound, where $z_{v(i)} = z_{<i}$ with auxiliary conditioning on $x$. Hence, the BSVI objective is a valid lower bound of $\ln p_\theta(x)$; we now refer to it as the BSVI bound where appropriate.

In the following two subsections, we address two additional aspects of the BSVI bound. First, we propose a method for ensuring that the BSVI bound is tighter than the Evidence Lower Bound achievable via SVI. Second, we provide an initial characterization of BSVI’s implicit \emph{sampling-importance-resampling} distribution.

\subsection{Buffer Weight Optimization}\label{sec:hyperweight}

Stochastic variational inference uses a series of gradient ascent steps to generate a final proposal distribution $q(z \giv \lambda_k)$. As evident from \Cref{fig:dependent_proposals}, the parameter $\lambda_k$ is in fact a random variable. The ELBO achieved via SVI, \emph{in expectation}, is thus
\begin{align}\label{eq:svi}
    \Expect_{q(z, \lambda_k \giv x)} \ln \frac{p_\theta(x, z)}{q_\phi(z \giv \lambda_k)}= \Expect_{q(z_{0:k} \giv x)} \ln w_k,
\end{align}
where the RHS re-expresses it in notation consistent with \cref{eq:bsvi}. We denote \cref{eq:svi} as the SVI bound. In general, the BSVI bound with uniform-weighting $\pi_i = 1 / (k + 1)$ is not necessarily tighter than the SVI bound. For example, if SVI's last proposal distribution exactly matches posterior $q_k(z) = p_\theta(z \giv x)$, then assigning equal weighting to across $(w_0, \ldots w_k)$ would make the BSVI bound looser. 
%\s{maybe a better counterexample is when $w_k$ gives the exact posterior, but $w_0$ doesn't. then mixing hurts}

In practice, we observe the BSVI bound with uniform-weighting to consistently achieve a tighter lower bound than SVI's last proposal distribution. We attribute this phenomenon to the effectiveness of variance-reduction from averaging multiple importance weights---even when these importance weights are generated from dependent and non-identical proposal distributions. 

To guarantee that the BSVI is tighter than the SVI bound, we propose to optimize the buffer weight $\pi$. This guarantees a tighter bound,
\begin{align}\label{eq:bufferweight}
\max_\pi \Expect_{q(z_{0:k} \giv x)} \ln \sum_{i=0}^k \pi_i w_i \ge \Expect_{q(z_{0:k} \giv x)} \ln w_k,
\end{align}
since the SVI bound is itself a special case of the BSVI bound when $\pi = (0, \ldots, 0, 1)$. It is worth noting that \cref{eq:bufferweight} is concave with respect to $\pi$, allowing for easy optimization of $\pi$.

Although $\pi$ is a local variational parameter, we shall, for simplicity, optimize only a single global $\pi$ that we update with gradient ascent throughout the course of training. As such, $\pi$ is jointly optimized with $\theta$ and  $\phi$.

\subsection{Dependence-Breaking via Double-Sampling}

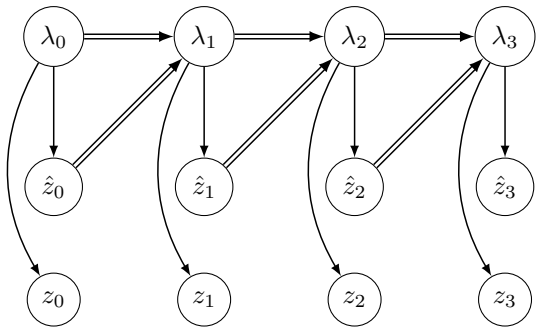
\begin{figure}[!h]
\centering
\begin{tikzpicture}
% \node[none] (x) {$x$};
\node[circ, xshift=+0cm] (l0) {$\lambda_0$};
\node[circ, xshift=+2cm] (l1) {$\lambda_1$};
\node[circ, xshift=+4cm] (l2) {$\lambda_2$};
\node[circ, xshift=+6cm] (l3) {$\lambda_3$};
\node[circ, xshift=+0cm, yshift=-2cm] (zh0) {$\hat{z}_0$};
\node[circ, xshift=+2cm, yshift=-2cm] (zh1) {$\hat{z}_1$};
\node[circ, xshift=+4cm, yshift=-2cm] (zh2) {$\hat{z}_2$};
\node[circ, xshift=+6cm, yshift=-2cm] (zh3) {$\hat{z}_3$};
\node[circ, xshift=+0cm, yshift=-3.5cm] (z0) {$z_0$};
\node[circ, xshift=+2cm, yshift=-3.5cm] (z1) {$z_1$};
\node[circ, xshift=+4cm, yshift=-3.5cm] (z2) {$z_2$};
\node[circ, xshift=+6cm, yshift=-3.5cm] (z3) {$z_3$};
% \node[circ, xshift=+2cm, yshift=-1.5cm] (zh0) {$\hat{z}_0$};
% \node[circ, xshift=+4cm, yshift=-1.5cm] (zh1) {$\hat{z}_1$};
% \node[circ, xshift=+6cm, yshift=-1.5cm] (zh2) {$\hat{z}_2$};
% \node[circ, xshift=+2cm, yshift=-3cm] (z0) {$z_0$};
% \node[circ, xshift=+4cm, yshift=-3cm] (z1) {$z_1$};
% \node[circ, xshift=+6cm, yshift=-3cm] (z2) {$z_2$};
\draw
% (x) edge [doubleconnect] (l0)
(l0) edge [doubleconnect] (l1)
(l1) edge [doubleconnect] (l2)
(l2) edge [doubleconnect] (l3)
(zh0) edge [doubleconnect] (l1)
(zh1) edge [doubleconnect] (l2)
(zh2) edge [doubleconnect] (l3)
(l0) edge [connect] (zh0)
(l1) edge [connect] (zh1)
(l2) edge [connect] (zh2)
(l3) edge [connect] (zh3)
(l0) edge [rconnect] (z0)
(l1) edge [rconnect] (z1)
(l2) edge [rconnect] (z2)
(l3) edge [rconnect] (z3)
;
\end{tikzpicture}
\caption{Graphical model for double sampling. Notice that the samples $z_{0:k}$ are now independent given $\lambda_{0:k}$ and $x$. Again the dependence on $\theta$ is omitted for notational simplicity.
% \rs{Rmbr to change captions here, and elsewhere.}
}
\label{fig:double_sampling}
\end{figure}

As observed in \cite{hoffman2013stochastic}, taking the gradient of the log-likelihood with respect to $\theta$ results in the expression
\begin{align}
    \nabla_\theta \ln p_\theta(x) = \Expect_{p_\theta(z \giv x)} \nabla_\theta \ln p_\theta(x, z).
\end{align}
We note that gradient of the ELBO with respect to $\theta$ results in a similar expression
\begin{align}
    \nabla_\theta \text{ELBO}(x) = \Expect_{q_\phi(z \giv x)} \nabla_\theta \ln p_\theta(x, z).
\end{align}
As such, the ELBO gradient differs from log-likelihood gradient only in terms of the distribution applied by the expectation operator. To approximate the log-likelihood gradient, we wish to set $q_\phi(z \giv x)$ close to $p_\theta(z \giv x)$ under some divergence.

We now show what results from computing the gradient of the BSVI objective.
\begin{restatable}{lemma}{LemBSVIGrad}\label{thm:bsvigrad}
The BSVI gradient with $\theta$ is
\begin{align}
    \nabla_\theta \text{BSVI}(x) &= \Expect_{q_\sir(z \giv x)} \nabla_\theta \ln p_\theta(x, z),
\end{align}
where $q_\sir$ is a sampling-importance-resampling procedure defined by the generative process
\begin{align}
    z_{0:k} &\sim q(z_{0:k} \giv x)\\
    i &\sim r(i \giv z_{0:k}) \\
    z &\gets z_i,
\end{align}
and $r(i \giv z_{0:k}) = ({\pi_i w_i})/({\sum_j \pi_j w_j})$ is a probability mass function over $\set{0, \ldots, k}$. 
\end{restatable}
A detailed proof is provided in \cref{app:proofs}.

A natural question to ask is whether BSVI's $q_\sir$ is closer to the posterior than $q_k$ in expectation. To assist in this analysis, we first characterize a particular instance of the Generalized IWAE bound when $(z_1, \ldots, z_k)$ are independent but \emph{non-identically distributed}.

\begin{restatable}{theorem}{ThmIndGIWAE}\label{thm:indgiwae}
When $q(z_{0:k}) = \prod_i q_i(z_i)$, the implicit distribution $q_\sir(z)$ admits the inequality
\begin{align}
\Expect_{q_\sir(z)}\ln \frac{p_\theta(x, z)}{q_\sir(z)} 
\ge \Expect_{q(z_{0:k})} \ln \sum_{i=0}^k \pi_i w_i \\
= \Expect_{q(z_{0:k})} \ln \sum_{i=0}^k \pi_i \frac{p_\theta(x, z)}{q_i(z_i)}.
\end{align}
\end{restatable}
%\s{add the inequality wrt to $w_k$?} \rs{inequality only possible with $\pi^*$.}
%\s{for the independent, non-identically distributed, can we show strict improvement over svi under some conditions? like in boosting variational inference. i guess if the ``last'' proposal is perfect, then there is no point in mixing them, so wont be possible in general..}

\Cref{thm:indgiwae} extends the analysis by \cite{cremer2017reinterpreting} from the i.i.d. case (i.e. the standard IWAE bound) to the non-identical case (proof in \cref{app:proofs}). It remains an open question whether the inequality holds for the non-independent case. 
%\rs{Should I even say this?}

Since the BSVI objective employs dependent samples, it does not fulfill the conditions for \cref{thm:indgiwae}. To address this issue, we propose a variant, \emph{BSVI with double-sampling} (BSVI-DS), that breaks dependency by drawing two samples at each SVI step: $\hat{z}_i$ for computing the SVI gradient update and $z_i$ for computing the BSVI importance weight $w_i$. 
%\jw{Shouldn't this be $z_i \perp z_j \giv (\lambda_{0:k}$? Just observing $\hat{z}_{0:k}$ doesn't break dependency completely} 
%\rs{I think the following statement is true: conditioning on $\hat{z}_{1:k}, x$ (...and $\theta, \phi$) implies conditioning on $\lambda_{0:k}$.}
The BSVI-DS bound is thus
\begin{align}
    \Expect_{q(\hat{z}_{<k}\giv x)}\paren{\Expect_{q({z}_{0:k} \giv \hat{z}_{<k}, x)} \ln \sum_{i=0}^k \pi_i \frac{p_\theta(x, z)}{q(z_i \giv \hat{z}_{<k}, x)}},
\end{align}
where $q({z}_{0:k} \giv \hat{z}_{<k}, x) = \prod_i q(z_i \giv \hat{z}_{<k}, x)$ is a product of independent but non-identical distributions when conditioned on $(\hat{z}_{<k}, x)$. Double-sampling now allows us to make the following comparison.

\begin{restatable}{corollary}{CorBSVIDS}\label{thm:bsvi-ds}
Let $q_k = q(z_k \giv \hat{z}_{<i}, x)$ denote the proposal distribution found by SVI. For any choice of $(\hat{z}_{<i}, x)$, the distribution $q_\sir$ implied by BSVI-DS (with optimal weighting $\pi^*$) is at least as close to $p_\theta(z \giv x)$ as $q_k$,
\begin{align}
\KL{q_\sir}{p_\theta(z \giv x)} \le \KL{q_k}{p_\theta(z \giv x)},
\end{align}
as measured by the Kullback-Leibler divergence.
\end{restatable}

\cref{thm:bsvi-ds} follows from \cref{thm:indgiwae} and that the BSVI-DS bound under optimal $\pi^*$ is no worse than the SVI bound. Although the double-sampling procedure seems necessary for inequality in \cref{thm:bsvi-ds} to hold, in practice we do not observe any appreciable difference between BSVI and BSVI-DS.

\section{Computational Considerations}

Another important consideration is the speed of training the generative model with BSVI versus SVI. Since BSVI reuses the trajectory of weights $(w_0, \ldots, w_k)$ generated by SVI, the forward pass incurs the same cost. The backwards pass for BSVI, however, is $O(k)$ %\jw{``for $K$ SVI steps''?} 
for $k$ SVI steps---in contrast to SVI's $O(1)$ cost. To make the cost of BSVI's backwards pass $O(1)$, we highlight a similar observation from the original IWAE study \cite{burda2015importance} that the gradient can be approximated via Monte Carlo sampling
\begin{align}
    \nabla_\theta \text{BSVI}(x) \approx \frac{1}{m} \sum_{i=1}^m \nabla_\theta \ln p_\theta(x, z^{(i)}),
\end{align}
where $z^{(i)}$ is sampled from BSVI's implicit distribution $q_\sir(z \giv x)$. We denote this as training \emph{BSVI with sample-importance-resampling} (BSVI-SIR). Setting $m = 1$ allows variational autoencoder training with BSVI-SIR to have the same wall-clock speed as training with SVI.

\section{Experiments}
% PARTIAL TABLE
\begin{table*}[!h]
\small
\centering
\caption{Test set performance on the Omniglot dataset. Note that $k = 9$ and $k' = 10$ (see \cref{sec:setup}). We approximate the log-likelihood with BSVI-$500$ bound (\cref{app:likelihood}). We additionally report the SVI-$500$ bound (denoted ELBO*) along with its KL and reconstruction decomposition.}
\label{table:omni_test}
\begin{center}
\begin{tabular}{l|c|ccc}
\bf{Model} & \bf{Log-likelihood} & \bf{ELBO*} & \bf{KL*} & \bf{Reconstruction*} \\
\hline
VAE                & -89.83 $\pm$ 0.03      & -89.88 $\pm$ 0.02      & 0.97 $\pm$ 0.13 & 88.91 $\pm$ 0.15 \\
IWAE-$k'$          & -89.02 $\pm$ 0.05      & -89.89 $\pm$ 0.06      & 4.02 $\pm$ 0.18 & 85.87 $\pm$ 0.15 \\
SVI-$k'$           & -89.65 $\pm$ 0.06      & \bf{-89.73 $\pm$ 0.05} & 1.37 $\pm$ 0.15 & 88.36 $\pm$ 0.20 \\
BSVI-$k$-SIR       & \bf{-88.80 $\pm$ 0.03} & -90.24 $\pm$ 0.06      & 7.52 $\pm$ 0.21 & \bf{82.72 $\pm$ 0.22} \\
\end{tabular}
\end{center}
\end{table*}

\begin{table*}[!h]
\small
\centering
\caption{Test set performance on the grayscale SVHN dataset.} 
\label{table:svhn_test}
\begin{center}
\begin{tabular}{l|c|ccc}
\textbf{Model} & \textbf{Log-likelihood} & \textbf{ELBO*} & \textbf{KL*} & \textbf{Reconstruction*} \\
\hline
VAE                & -2202.90 $\pm$ 14.95     & -2203.01 $\pm$ 14.96     & 0.40 $\pm$ 0.07  & 2202.62 $\pm$ 14.96 \\
IWAE-$k'$          & -2148.67 $\pm$ 10.11     & -2153.69 $\pm$ 10.94     & 2.03 $\pm$ 0.08  & 2151.66 $\pm$ 10.86 \\
SVI-$k'$           & -2074.43 $\pm$ 10.46     & -2079.26 $\pm$ 9.99      & 45.28 $\pm$ 5.01 & 2033.98 $\pm$ 13.38 \\
BSVI-$k$-SIR       & \bf{-2059.62 $\pm$ 3.54}      & \bf{-2066.12 $\pm$ 3.63}      & 51.24 $\pm$ 5.03 & \bf{2014.88 $\pm$ 5.30}
\end{tabular}
\end{center}
\end{table*}

\begin{table*}[!h]
\small
\centering
\caption{Test set performance on the FashionMNIST dataset.}
\label{table:fmnist_test}
\begin{center}
\begin{tabular}{l|c|ccc}
\bf{Model} & \bf{Log-likelihood} & \bf{ELBO*} & \bf{KL*} & \bf{Reconstruction*} \\
\hline
VAE & -1733.86 $\pm$ 0.84      & -1736.49 $\pm$ 0.73      & 11.62 $\pm$ 1.01 & 1724.87 $\pm$ 1.70 \\
IWAE-$k'$ & -1705.28 $\pm$ 0.66      & -1710.11 $\pm$ 0.72      & 33.04 $\pm$ 0.36 & 1677.08 $\pm$ 0.70 \\
SVI-$k'$ & -1710.15 $\pm$ 2.51      & -1718.39 $\pm$ 2.13      & 26.05 $\pm$ 1.90 & 1692.34 $\pm$ 4.03 \\
BSVI-$k$-SIR & \bf{-1699.44 $\pm$ 0.45}      & \bf{-1707.00 $\pm$ 0.49}      & 41.48 $\pm$ 0.12 & \bf{1665.52 $\pm$ 0.41}
\end{tabular}
\end{center}
\end{table*}

\subsection{Setup}\label{sec:setup}

%\rs{Move BSVI variants to appendix. Reflect that in the writing. Remember to actually put the full table in the appendix}

%\rs{Describe experimental set-up. Optimizer, architecture, datasets, hyperparameters, etc. If too long, we'll move some stuff to appendix.}
% \rs{Describe that we compare against VAE, SVI, IWAE, double-sampling, pi-optimization.}
% \jw{Done.}

%\rs{Emphasize the off-by-one issue}
%\jw{This sentence seems too wordy with repeated terms}

We evaluated the performance of our method by training variational autoencoders with BSVI-SIR with buffer weight optimization (BSVI-SIR-$\pi$)) on the dynamically-binarized Omniglot, grayscale SVHN datasets, and FashionMNIST (a complete evaluation of all BSVI variants is available in \cref{app:all_performances}).
Our main comparison is against the SVI training procedure (as described in \cref{alg:bsvi}). We also show the performance of the standard VAE and IWAE training procedures. Importantly, we note that %\textbf{Important Note}:\jw{is this bolded message necessary? We could just say ``Note that we have chosen ...''} 
we have chosen to compare SVI-$k'$ and IWAE-$k'$ trained with $k' = 10$ against BSVI-$k$-SIR trained with $k = 9$ SVI steps. This is because that BSVI-$k$-SIR generates $k + 1$ importance weights.

For all our experiments, we use the same architecture as \cite{kim2018semi} (where the decoder is a PixelCNN) and train with the AMSGrad optimizer \cite{reddi2018adam}.
%\jw{Rephrase as: ``For all our experiments, we used the same architecture (autoregressive PixelCNN) as \cite{kim2018semi} and trained with ...''?} 
For grayscale SVHN, we follow \cite{kingma2016improved} and replaced \cite{kim2018semi}'s bernoulli observation model with a discretized logistic distribution model with a global scale parameter. Each model was trained for up to 200k steps with early-stopping based on validation set performance. For the Omniglot experiment, we followed the training procedure in \cite{kim2018semi} and annealed the KL term multiplier \cite{sonderby2016ladder,bowman2015generating} during the first $5000$ iterations.
We replicated all experiments four times and report the mean and standard deviation of all relevant metrics. For additional details, refer to \cref{app:setup}

\subsection{Log-Likelihood Performance}

%\rs{Point out how log-likelihood is approximated.}
%\rs{Todo: we might have to show performance of IWAE-SIR if we want to be really complete. However, important to stress that the goal is not out-performing IWAE.}

% \rs{Show Omniglot and SVHN tables} \jw{Done.}

%\rs{Comment on how we're not making z so the ELBO is tight.}

%\rs{Discuss \cref{table:omni_test}. Comment on how all BSVI variants outperform SVI/VAE/IWAE. Comment on how the amortized ELBO is not a good reflection of the log-likelihood. The unamortized ELBO indicate how there's a gap that the amortized inference model is unable to overcome. Gap is present in both the training and test set, indicating it is not simply an issue of the amortized inference model overfitting, but a fundamental failure to match the $\lambda^*$ during training. Furthermore, comment on how the unamortized ELBO is also bad since it uses a Gaussian variational family. This suggests that the true posterior found by BSVI tend not to be Gaussian. Comment on how the reconstructions are much better for BSVI than the other models, suggesting higher usage of the latent space by BSVI. Furthermore, comment on the usefulness of DS/SIR/pi.}
%\jw{Do we need to comment on how BSVI has a better usage of the latent space here? Section 6.4 covers this in more detail, so this seems redundant. The value of pi is discussed in Section 6.5 and SIR in Section 6.3. So I'm only discussing the value of DS here}

For all models, we report the log-likelihood (as measured by BSVI-$500$). We additionally report the SVI-$500$ (ELBO*) bound along with its decomposition into rate (KL*) and distortion (Reconstruction*) components \cite{alemi2018fixing}. We highlight that KL* provides a \emph{fair} comparison of the rate achieved by each model without concern of misrepresentation caused by the amortized inference suboptimality.

\textbf{Omniglot}. \cref{table:omni_test} shows that BSVI-SIR outperforms SVI on the test set log-likelihood. BSVI-SIR also makes greater usage of the latent space (as measured by the lower Reconstruction*). 
% \jw{Why does this sentence start with "While"? Could we just say "The greater amortization gap that BSVI experiences in cmoparison to SVI is attributable to ..."? }. 
% \rs{noted. Changed.}
Interestingly, BSVI-SIR's log-likelihoods are noticeably higher than its corresponding ELBO*, suggesting that BSVI-SIR has learned posterior distributions not easily approximated by the Gaussian variational family when trained on Omniglot.

%Note that the gap between log-likelihood and ELBO values exists for both training and test sets, which indicates that this is not simply an issue of overfitting. For the amortized ELBO, this indicates a fundamental failure of the inference model to match the optimal local variational parameter $\lambda^*$ during training. For the unamortized ELBO, whose variational family is restricted to Gaussian, this indicates that the true posterior found by BSVI is likely not Gaussian.

%\rs{Basically make similar comments, but this time recognizing that SVI itself is capable of achieving a pretty decent score this time. The that fact SVI and BSVI outperform the other models indicate that SVHN is in a regime where minimizing the amortization gap plays a critical role. Say that this is reflect in the svi-diff table, which shows consistently positive svi-diff. Comment on usefulness of DS/SIR/pi.}

\textbf{SVHN}. \cref{table:svhn_test} shows that BSVI-SIR outperforms SVI on test set log-likelihood. We observe that both BSVI-SIR and SVI significantly outperform both VAE and IWAE on log-likelihood, ELBO*, and Reconstruction*, demonstrating the efficacy of iteratively refining the proposal distributions found by amortized inference model during training.

\textbf{FashionMNIST}. \cref{table:fmnist_test} similarly show that BSVI-SIR outperforms SVI on test set log-likelihood. Here, BSVI achieves significantly better Reconstruction* as well as achieving higher ELBO* compared to VAE, IWAE, and SVI.
%\jw{Maybe ``..significantly better rate and reconstruction loss''? to avoid using ``higher'' for rate, but ``better'' for reconstruction.}
%\rs{We can't do that. Higher rate is not necessarily better.}
%\jw{maybe "better reconstruction \textbf{loss}"? I'm worried the reader might think this is referring to the quality of reconstruction and ask for image samples} 

In \cref{table:omni_full_test,table:svhn_full_test,table:fmnist_full_test} (\cref{app:all_performances}), we also observe that the use of double sampling and buffer weight optimization does not make an appreciable difference than their appropriate counterparts, demonstrating the efficacy of BSVI even when the samples $(z_{0:k})$ are statistically dependent and the buffer weight is simply uniform.

%\jw{Is it necessary to discuss the usefulness of DS/SIR/pi again here?}
%\jw{How much of what is discussed for Omniglot do I need to restate here?  (i.e. what the gap between amortized/unamortized ELBO suggests)}

%\loremipsum

\subsection{Stochastic Gradient as Regularizer}
%\jw{Regularization via Stochastic Gradient?}
\begin{figure}[!h]
\centering
\includegraphics[width=.9\columnwidth]{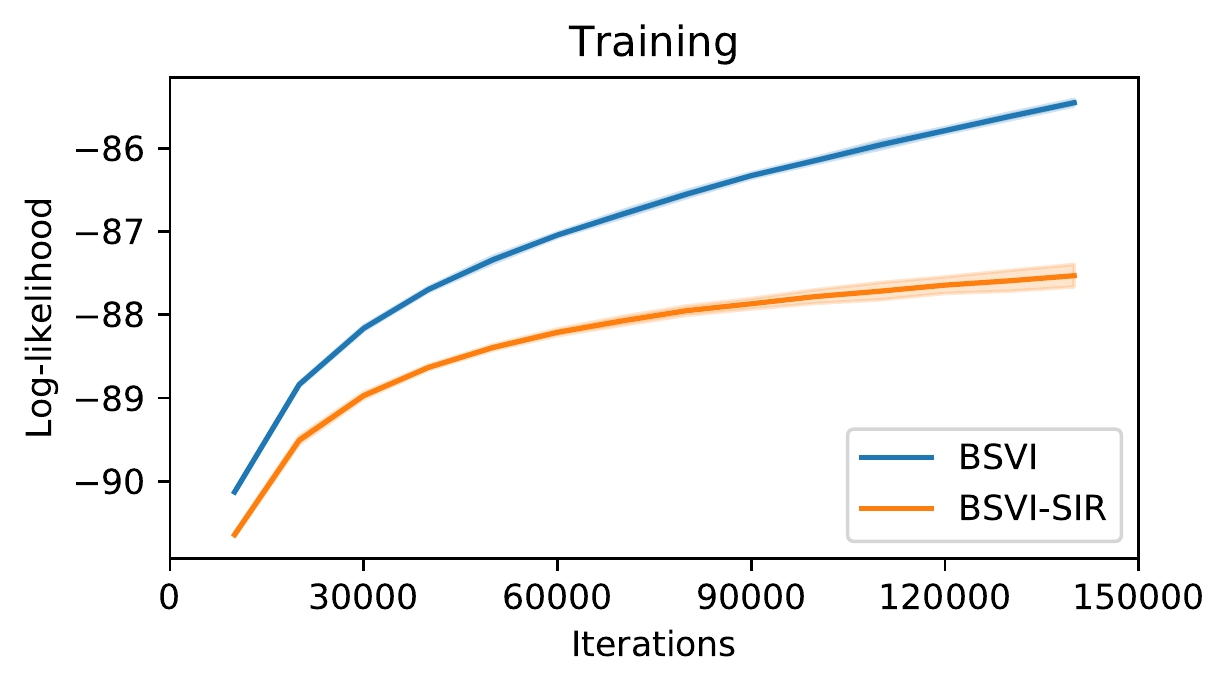}
\includegraphics[width=.9\columnwidth]{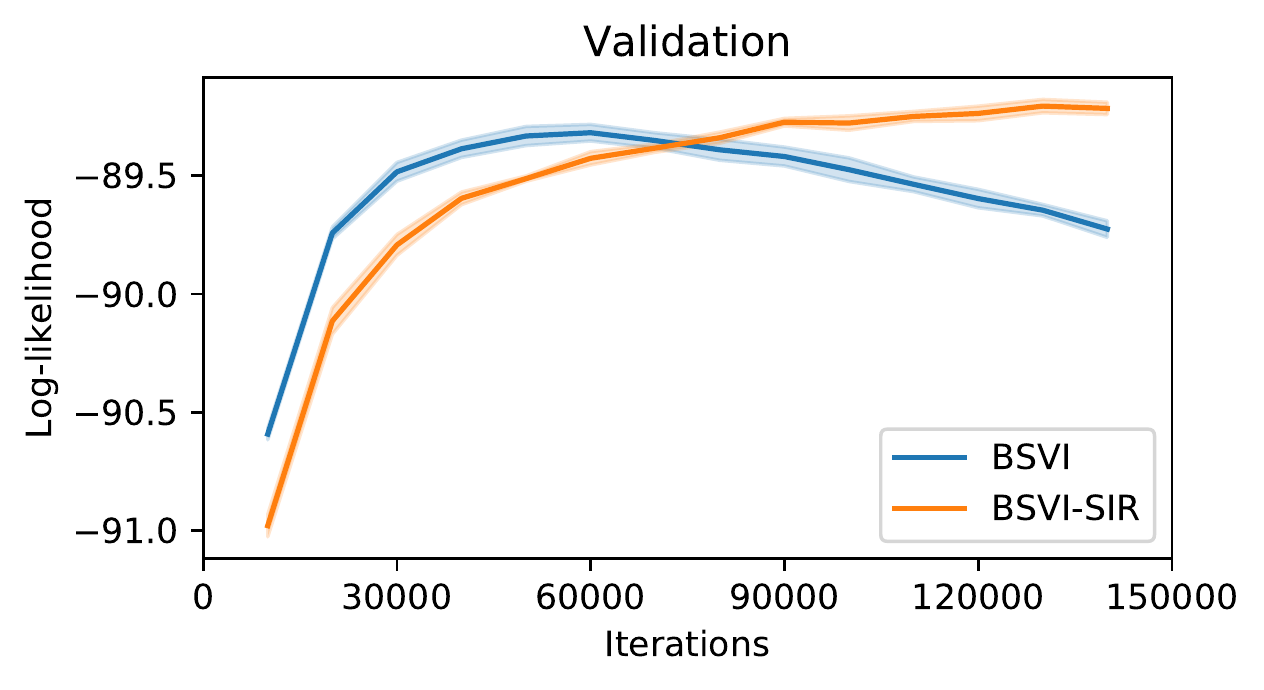}
\caption{Performance comparison between BSVI and BSVI-SIR on training (top) and validation (bottom) sets for Omniglot. Although BSVI achieves lower training loss, BSVI-SIR avoids overfitting and performs better on the test set.}
%\rs{Change the plot layout: instead of side-by-side, put one on top of the other and make each plot longer and slightly taller. This way the plots are a little bigger, and the font is a little easier to read.}
% \jw{Thoughts on the updated plot?}
% \rs{Change y-axis to Log-likelihood. Change x-axis to Iterations. Use actual number (clearly you don't actually mean 15 steps only).}
% \jw{Updated. Let me know what you think}
% \rs{Remember to use negative numbers and actually flip the plot for LL}
% \rs{Change the plots aspect ratio so the plot box is longer.}
% \jw{Updated}
%\jw{Recreate the figure with subplot(2,1)}
\label{fig:ad_vs_sir}
\end{figure}

%\rs{Show Omniglot overfitting. Use this to conclude that SIR does in fact incur a cost from the higher signal to noise ratio \emph{looking only at the training loss}. However, this does not translate to better performance on the test set. Indicate that the observation here is consistent with some of the analysis done in AIR. This suggest that could in fact benefit from regularization. Although AIR suggests that regularization can take the form of weakening the inference model, it seems that an orthogonal procedure is increasing gradient variance. While these are both possible sources of regularization, the ultimate choice of how to regularize a generative model remains an opening question.}

Interestingly, \cref{table:omni_full_test} shows that BSVI-SIR can outperform BSVI on the test set despite having a higher variance gradient. We show in \Cref{fig:ad_vs_sir} that this is the result of BSVI overfitting the training set. The results demonstrate the regularizing effect of having noisier gradients and thus provide informative empirical evidence to the on-going discussion about the relationship between generalization and the gradient signal-to-noise ratio in variational autoencoders \cite{rainforth2018tighter,shu2018amortized}.

%\loremipsum

\subsection{Latent Space Visualization}

% \rs{Show Omniglot visualization.} \jw{Done.}

\begin{figure}[!h]
\centering
\includegraphics[width=\columnwidth]{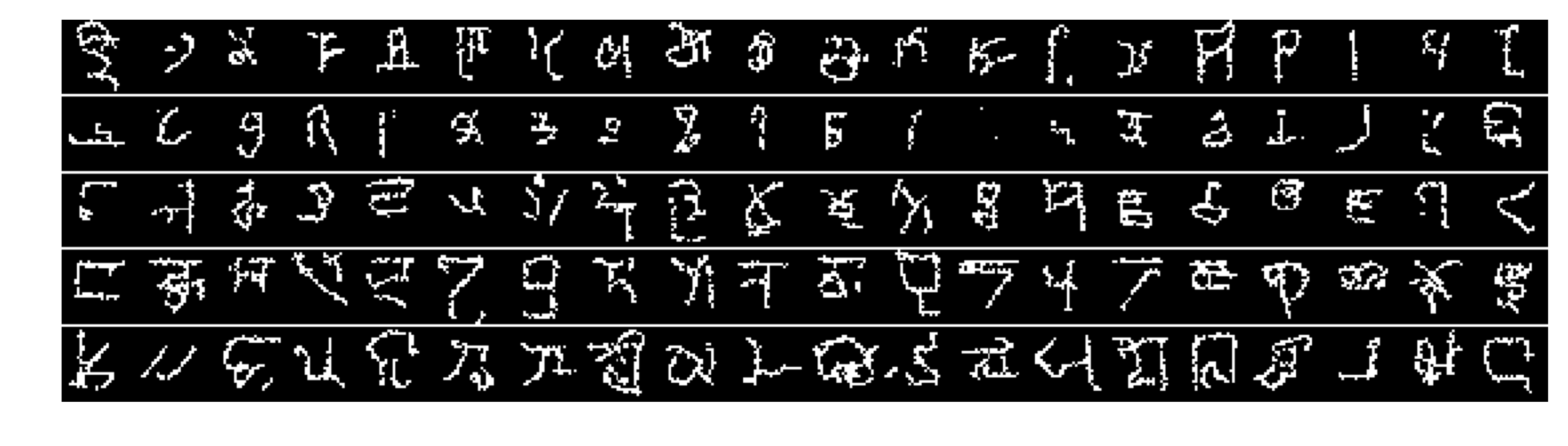}
\includegraphics[width=\columnwidth]{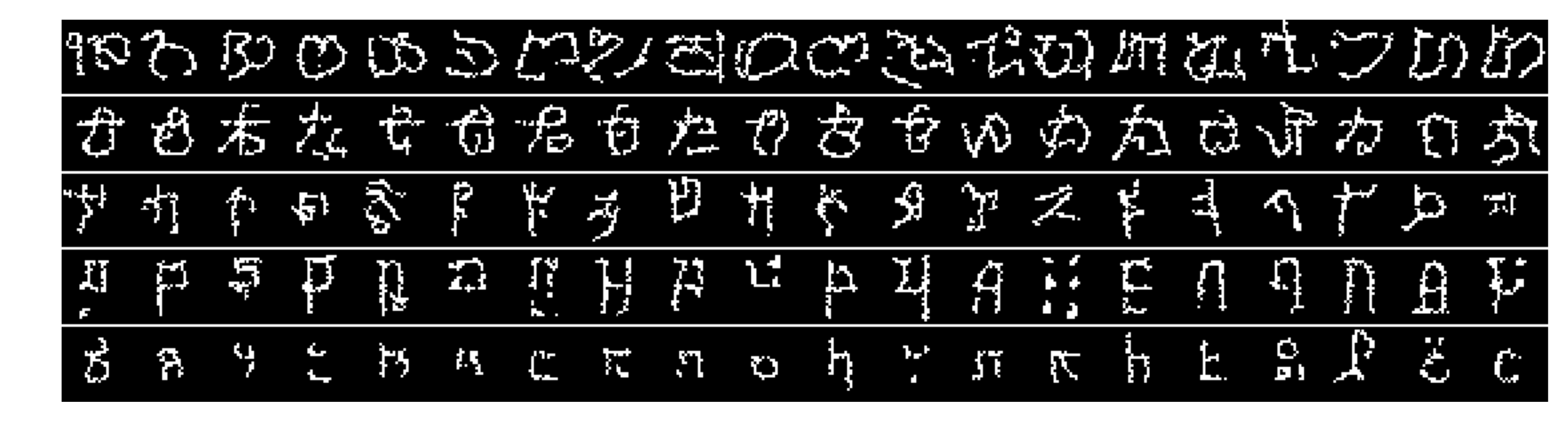}
\caption{Visualization of images sampled from decoder trained using SVI (top) and BSVI-SIR (bottom). Each row represents a different $z$ sampled from the prior. Conditioned on $z$, 20 images $x^{(1:20)} \sim p_\theta(x \giv z)$ are then sampled from the PixelCNN decoder.}
\label{fig:omniviz}
\end{figure}

%\rs{TODO: Comment on how the KL is higher for BSVI than SVI. Explain the procedure for building the figure. Explain how the visualization indicates that BSVI is making more use of the latent space. Provide a vague hypothesis for why this might be the case: there's something special about the implicit distribution used by BSVI to train the model that is characteristically different from SVI and IWAE distributions. One possibility is that this is actually induced by the decoupled nature of the training. It will be quite interesting to check whether we continue to see this if we move to end-to-end training of BSVI.}

\cref{table:omni_test} shows that the model learned by BSVI-SIR training has better Reconstruction* than SVI, indicating greater usage of the latent variable for encoding information about the input image. We provide a visualization of the difference in latent space usage in \Cref{fig:omniviz}. Here, we sample multiple images conditioned on a fixed $z$. Since BSVI encoded more information into $z$ than SVI on the Omniglot dataset, we see that the conditional distribution $p_\theta(x \giv z)$ of the model learned by BSVI has lower entropy (i.e. less diverse) than SVI.

%We conjecture that the implicit proposal distribution induced by BSVI is characteristically different from those induced by SVI and IWAE. %\jw{Should we mention the ``decoupling'' argument here? Seems rather handwavy.} 

%\rs{Kinda concerned that people will ask for SVHN visualization---they don't look good. Proposal sampling of SVHN has always been challenging. On my TODO list after submission deadline is to return to using ladder VAEs.}

%\rs{Make small observation note about the KL and then move on.} \jw{Not sure what you have in mind here.} 

\subsection{Analysis of Training Metrics}
\begin{figure}[!ht]
\centering
\begin{subfigure}[b]{\columnwidth}
\centering
\includegraphics[width=\columnwidth]{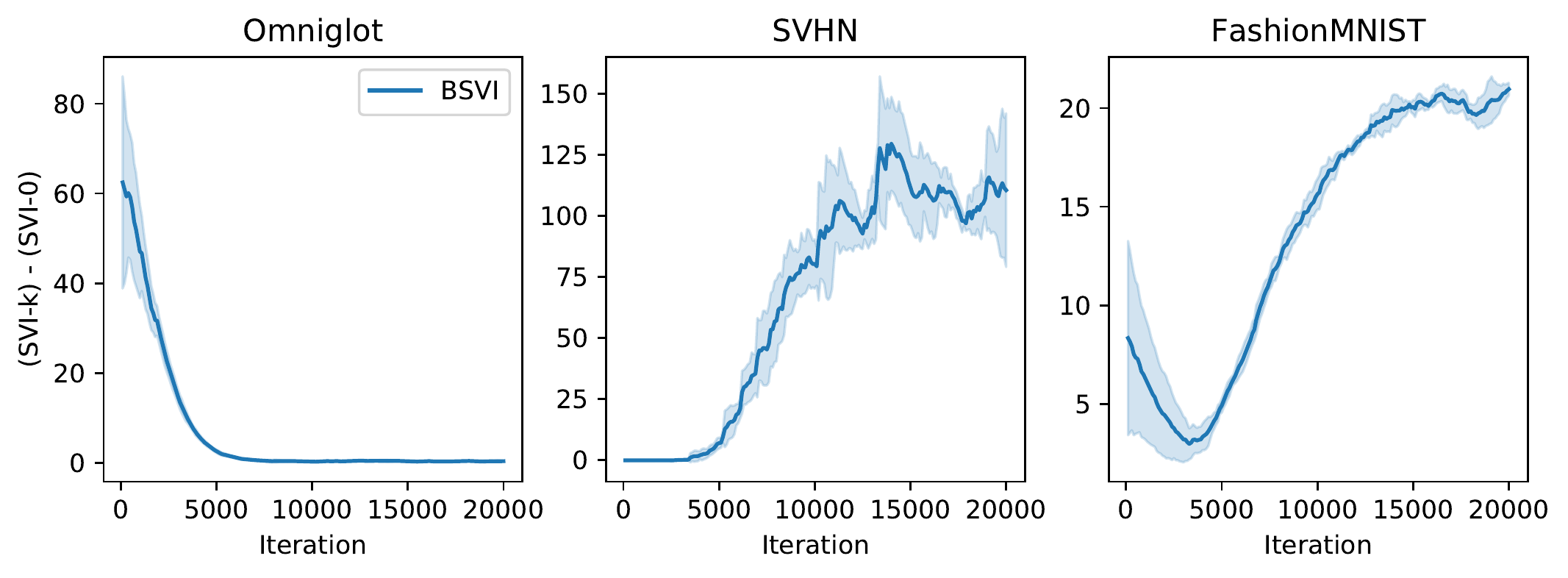}
\caption{Difference between lower bounds achieved by $q_k$ (SVI-$k$) and $q_0$ (SVI-$0$) during training.}
\label{fig:compare_svi_diff}
\end{subfigure}

\begin{subfigure}[b]{\columnwidth}
\centering
\includegraphics[width=\columnwidth]{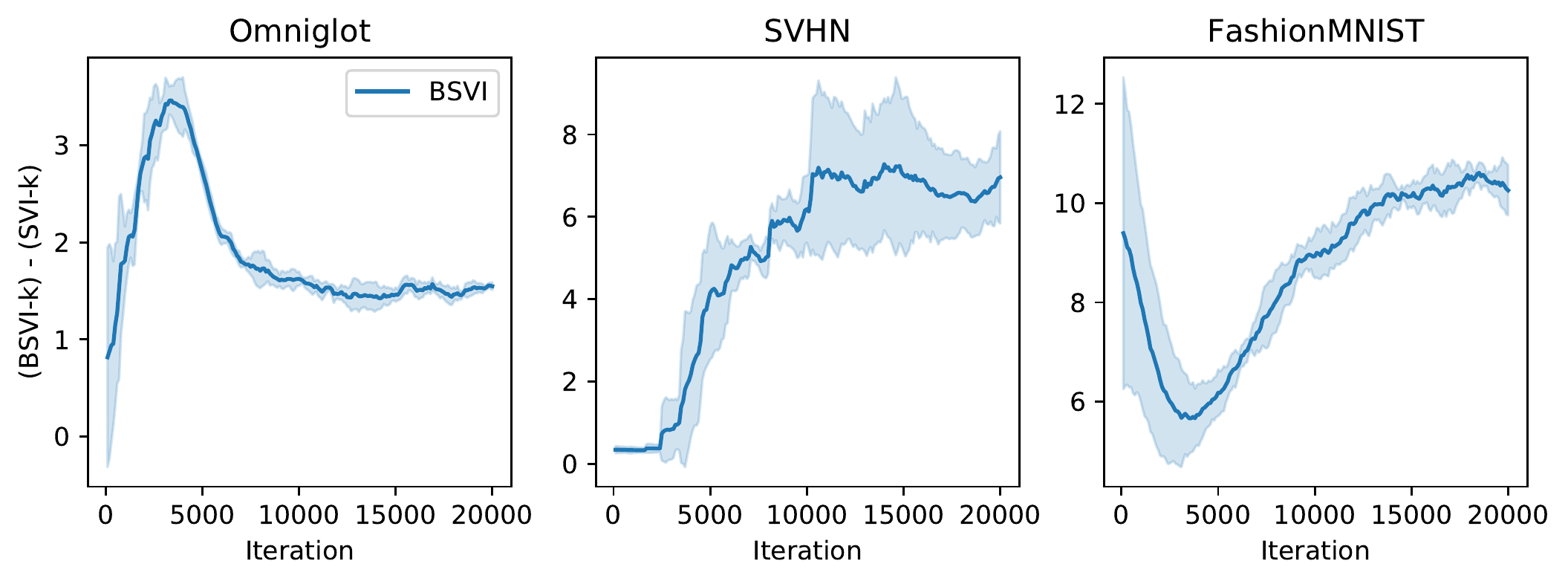}
\caption{Difference between the BSVI-$k$ bound and SVI-$k$ bound during training.}
\label{fig:compare_buff_diff}
\end{subfigure}

\begin{subfigure}[b]{\columnwidth}
\centering
\includegraphics[width=\columnwidth]{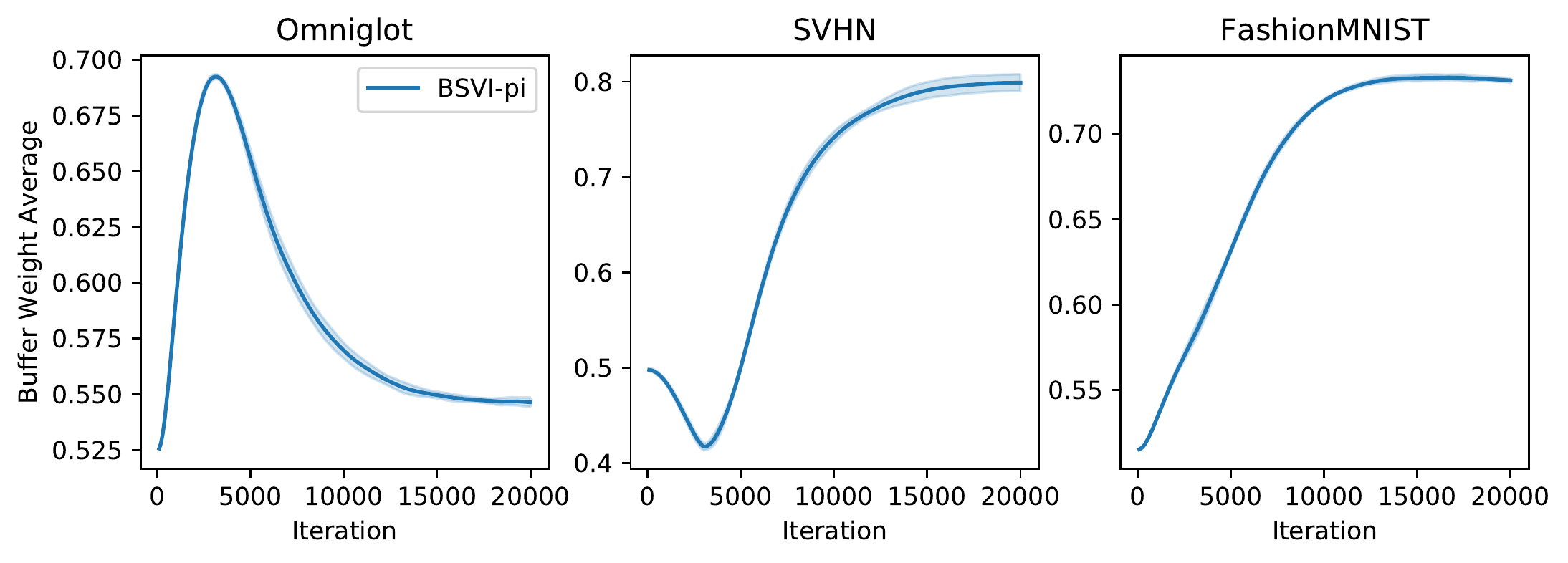}
\caption{Plot of the buffer weight average (defined as $\Expect_{\pi(i)} i / k$) during training when the buffer weight is optimized.}
\label{fig:compare_pi}
\end{subfigure}

% \begin{subfigure}[b]{\columnwidth}
% \centering
% \includegraphics[width=\columnwidth]{images/compare_kl.pdf}
% \caption{TODO: KL}
% \label{fig:compare_kl}
% \end{subfigure}

\caption{Plots of metrics during BSVI-$k$ training, where $k = 9$. Since BSVI-$k$ uses SVI-$k$ as a subroutine, it is easy to check how the BSVI-$k$ bound compares against the SVI-$k$ and the amortized ELBO (SVI-$0$) bounds on a random mini-batch at every iteration during training.}
\end{figure}
%\input{figures/training_metrics_wide.tex}

% \rs{Provide a picture of how pi changes throughout training.} \jw{Done.}
% \input{figures/compare_pi.tex}

% \rs{Provide a picture of SVI-DIFF} \jw{Done.}
% \input{figures/compare_svi_diff.tex}

% \rs{Provide a picture of buff-diff} \jw{Done.}
% \input{figures/compare_buff_diff.tex}

%\rs{Provide picture of KL. We'll make a big deal out of how KL goes to zero for omniglot and discuss the KL ramping stuff.}
% \input{figures/compare_kl.tex}

%\rs{Explain all three pictures.}

%\rs{If there's room: discuss how $\pi$ and SVI gaps and IWAE gaps changes over training. Say how pi implies that that SVI does find better proposal distributions}

%\rs{Comment on why pi doesn't really affect training. Two possible explanations: SIR-ELBO would be significantly higher than BSVI bound. Good approximation not necessarily the best optimization.}

%\rs{Comment about how I initialized pi. And that alternative is SVI-initialization.}

Recall that the BSVI-$k$ training procedure runs SVI-$k$ as a subroutine, and therefore generates the trajectory of importance weights $(w_0, \ldots, w_k)$. Note that $\ln w_0$ and $\ln w_k$ are unbiased estimates of the ELBO achieved by the proposal distribution $q_0$ (SVI-$0$ bound) and $q_k$ (SVI-$k$ bound) respectively. It is thus possible to monitor the health of the BSVI training procedure by checking whether the bounds adhere to the ordering
\begin{align}
    \text{BSVI-}k \ge \text{SVI-}k \ge \text{SVI-}0
\end{align}
in expectation. \Cref{fig:compare_svi_diff,fig:compare_buff_diff} show that this is indeed the case. Since Omniglot was trained with KL-annealing \cite{kim2018semi}, we see in \Cref{fig:compare_svi_diff} that SVI plays a negligible role once the warm-up phase (first $5000$ iterations) is over. In contrast, SVI plays an increasingly large role when training on the more complex SVHN and FashionMNIST datasets, demonstrating that the amortization gap is a significantly bigger issue in the generative modeling of SVHN and FashionMNIST. \Cref{fig:compare_buff_diff} further shows that BSVI-$k$ consistently achieves a better bound than SVI-$k$. When the buffer weight is also optimized, we see in \Cref{fig:compare_pi} that $\pi$ learns to upweight the later proposal distributions in $(q_0, \ldots, q_k)$, as measured by the buffer weight average $\Expect_{\pi(i)} i/k$. For SVHN, the significant improvement of SVI-$k$ over SVI-$0$ results in $\pi$ being biased significantly toward the later proposal distributions. Interestingly, although \Cref{fig:compare_pi} suggests that the optimal buffer weight $\pi^*$ can differ significantly from naive uniform-weighting, we see from \cref{table:omni_test,table:svhn_test} that buffer weight optimization has a negligible effect on the overall model performance.

\section{Conclusion}

In this paper, we proposed \emph{Buffered Stochastic Variational Inference} (BSVI), a novel way to leverage the intermediate importance weights generated by stochastic variational inference. We showed that BSVI is effective at alleviating inference suboptimality and that training variational autoencoders with BSVI consistently outperforms its SVI counterpart, making BSVI an attractive and simple drop-in replacement for models that employ SVI. One promising line of future work is to extend the BSVI training procedure with end-to-end learning approaches in \cite{kim2018semi,marino2018iterative}. Additionally, we showed that BSVI procedure is a valid lower bound and belongs to general class of importance-weighted (Generalized IWAE) bounds where the importance weights are statistically dependent.
Thus, it would be of interest to study the implications of this bound for certain MCMC procedures such as Annealed Importance Sampling \cite{neal2001annealed} and others.

\subsubsection*{Acknowledgements}

We would like to thank Matthew D. Hoffman for his insightful comments and discussions during this project. This research was supported by NSF (\#1651565, \#1522054, \#1733686), ONR (N00014-19-1-2145), AFOSR (FA9550-19-1-0024), and FLI.

\bibliography{main}
\bibliographystyle{unsrt}

\clearpage
\onecolumn
\appendix
\section{Proofs}\label{app:proofs}

\ThmGIWAE*
\begin{proof}To show the validity of this lower bound, note that
\begin{align}
\Expect_{q(z_0, \ldots, z_k)} \brac{\sum_i \pi_i \frac{p_\theta(x, z_i)}{q(z_i \giv z_{v(i)})}} 
&= \sum_i \pi_i \Expect_{q(z_0, \ldots, z_k)} \frac{p_\theta(x, z_i)}{q(z_i \giv z_{v(i)})} \\
&= \sum_i \pi_i \Expect_{q(z_{v(i)})} \Expect_{q(z_i \giv z_{v(i)})} \frac{p_\theta(x, z_i)}{q(z_i \giv z_{v(i)})} \\
&= \sum_i \pi_i \Expect_{q(z_{v(i)})} p_\theta(x)\\
&= p_\theta(x).
\end{align}
Applying Jensen's inequality shows that the lower bound in the theorem is valid.
\end{proof}

\LemBSVIGrad*
\begin{proof}
\begin{align}
    \nabla_\theta \text{BSVI}(x) &= \Expect_{q(z_{0:k} \giv x)} \nabla_\theta \ln \sum_{i=0}^k \pi_i w_i\\
    &= \Expect_{q(z_{0:k} \giv x)} \Expect_{r(i \giv z_{0:k})} \nabla_\theta \ln p_\theta(x, z_i),
\end{align}
The double-expectation can now be reinterpreted as
the \emph{sampling-importance-resampling} distribution $q_\sir$.
\end{proof}

\newpage
\ThmIndGIWAE*
\begin{proof}
Recall that the $q_\sir$ is defined by the following sampling procedure
\begin{align}
(z_0, \ldots, z_k) &\sim q(z_0, \ldots, z_k)\\
i &\sim r(i \giv z_{0:k})\\
z &\gets z_i,
\end{align}
where 
\begin{align}
r(i \giv z_{0:k}) = \frac{\pi_i w_i}{\sum_j \pi_j w_j} = \frac{\pi_i \frac{p(x, z_i)}{q(z_i \giv z_{<i})}}{\sum_j \pi_j \frac{p(x, z_i)}{q(z_j \giv z_{<j})}}
\end{align}

We first note that, for any distribution $r(z)$
\begin{align}
    r(z) = \int_a r(a) \delta_z(a) \d a = \Expect_{r(a)} \delta_z(a).
\end{align}
This provides an intuitive way of constructing the probability density function by reframing it as a sampling process (the expectation w.r.t. $r(a)$) paired with a filtering procedure (the dirac-delta $\delta_z(a)$). Thus, the density under $q_\sir$ is thus
\begin{align}
    q_\sir(z) = \Expect_{q_(z_{0:k})} \Expect_{r(i \giv z_{0:k})} \delta_z(z_i).
\end{align}
Additionally, we shall introduce the following terms 
\begin{align}
    \tp(z) &= p_\theta(x, z)\\
    \bw_i &= \frac{\pi_i w_i}{\sum_j \pi_j w_j}\\
    \bv_i &= \frac{w_i}{\sum_j \pi_j w_j}\\
    \bv_i(z) &= \frac{w(z)}{\pi_i w(z) + \sum_{-i} \pi_j w_j}.
\end{align}
for notational simplicity. Note that the density function $q_\sir$ can be re-expressed as
\begin{align}
q_\sir(z) 
&= \Expect_{q(z_{0:k})}\Expect_{r(i \giv  z_{0:k})}\delta_z(z_i)\\
&= \Expect_{q(z_{0:k})}\sum_i \pi_i \bv_i \delta_z(z_i)\\
&= \Expect_{\pi(i)}\Expect_{q_{z_{-i}}}\Expect_{q_i(z_i \giv z_{-i})} \bv_i \delta_z(z_i)\\
&= \Expect_{\pi(i)}\Expect_{q_{z_{-i}}} \bv_i(z) q_i(z \giv z_{-i}).
\end{align}
We now begin with the ELBO under $q_\sir(z)$ and proceed from there via
\begin{align}
\Expect_{q_\sir}(z) \ln \frac{\tp(z)}{q_\sir(z)}
&= -\uKL{q_\sir(z)}{\tp(z)}\\
%&= -\uKL{\Expect_{\pi(i)}\Expect_{q_{z_{-i}}} \bv_i(z) q_i(z \giv z_{-i})}{\tp(z)}\\
&= -\uKL{\Expect_{\pi(i)}\Expect_{q_{z_{-i}}} \bv_i(z) q_i(z \giv z_{-i})}{\tp(z)} \\
&\ge -\Expect_{\pi(i)}\Expect_{q_{z_{-i}}} \uKL{\bv_i(z) q_i(z \giv z_{-i})}{\tp(z)},
\end{align}
where we use Jensen's Inequality to exploit the convexity of the unnormalized Kullback-Leibler divergence $\uKL{\cdot}{\cdot}$. We now do a small change of notation when rewriting the unnormalized KL as an integral to keep the notation simple
\begin{align}
\Expect_{q_\sir}(z) \ln \frac{\tp(z)}{q_\sir(z)}
&\ge \Expect_{\pi(i)}\Expect_{q_{z_{-i}}} \int_{z_i} \bv_i q(z_i \giv z_{-i}) \ln \frac{\tp(z_i)}{\bv_i q(z_i \giv z_{-i})} \\
&= \Expect_{\pi(i)}\Expect_{q_{z_{-i}}}\Expect_{q(z_i \giv z_{-i})} \bv_i \ln \frac{\tp(z_i)}{\bv_i q(z_i \giv z_{-i})} \\
&= \Expect_{q(z_{0:k})} \sum_i \bw_i \ln \frac{\tp(z_i)}{\bv_i q(z_i \giv z_{-i})} \\
&= \Expect_{q(z_{0:k})} \sum_i \bw_i \ln \paren{\sum_j \pi_j w_j 
\cdot \frac{q(z_i \giv z_{<i})}{q(z_i \giv z_{-i})}} \\
&= \Expect_{q(z_{0:k})} \sum_i \bw_i \brac{\ln \paren{\sum_j \pi_j w_j} 
+ \ln \paren{\frac{q(z_i \giv z_{<i})}{q(z_i \giv z_{-i})}}} \\
\end{align}
If $z_{0:k}$ are independent, then it follows that $q(z_i \giv z_{<i}) = q(z_i \giv z_{-i}) = q(z_i)$. Thus,
\begin{align}
    \Expect_{q_\sir}(z) \ln \frac{\tp(z)}{q_\sir(z)} 
    &\ge \Expect_{q(z_{0:k})} \sum_i \bw_i \brac{\ln \paren{\sum_j \pi_j w_j} 
    + \ln \paren{\frac{q(z_i \giv z_{<i})}{q(z_i \giv z_{-i})}}} \\
    &= \Expect_{q(z_{0:k})} \sum_i \bw_i \ln \paren{\sum_j \pi_j w_j} \\
    &= \Expect_{q(z_{0:k})} \ln \paren{\sum_j \pi_j w_j}.
\end{align}
\end{proof}

\newpage
\section{Model Performance on Test and Training Data}
\label{app:all_performances}
Here we report various performance metrics for each type of model trained on the training set for both Omniglot and SVHN. As stated earlier, log-likelihood is estimated using BSVI-500, and ELBO* refers to the lower bound achieved by SVI-500 (i.e. $z \sim q_{500}$).
KL* and Reconstruction* are the rate and distortion terms for ELBO*, respectively.
\begin{align}
    \text{Log-likelihood}
    &= \Expect_{q(z_{0:500} \giv x)} \brac{\ln \sum_{i=0}^{500} \pi_i \frac{p_\theta(x, z_i)}{q(z_i \giv z_{<i}, x)}} \\
    \text{ELBO*} &= \underbrace{\Expect_{q_{500}} \brac{\ln p_\theta(x \giv z)}}_{\text{Reconstruction*}}
    + \underbrace{\uKL{q_{500}(z)}{p_\theta(z)}}_{\text{KL*}}
    % \text{ELBO} &= \Expect_{q_0} \brac{\ln p_\theta(x \giv z)} + \uKL{q_0(z)}{p_\theta(z)}
\end{align}

\begin{table*}[!h]
\small
\centering
\caption{Test set performance on the Omniglot dataset. Note that $k = 9$ and $k' = 10$ (see \cref{sec:setup}). We approximate the log-likelihood with BSVI-$500$ bound (\cref{app:likelihood}). We additionally report the SVI-$500$ bound (denoted ELBO*) along with its KL and reconstruction decomposition.}
\label{table:omni_full_test}
\begin{center}
\begin{tabular}{l|c|ccc}
\bf{Model} & \bf{Log-likelihood} & \bf{ELBO*} & \bf{KL*} & \bf{Reconstruction*} \\
\hline
VAE                & -89.83 $\pm$ 0.03      & -89.88 $\pm$ 0.02      & 0.97 $\pm$ 0.13 & 88.91 $\pm$ 0.15 \\
IWAE-$k'$          & -89.02 $\pm$ 0.05      & -89.89 $\pm$ 0.06      & 4.02 $\pm$ 0.18 & 85.87 $\pm$ 0.15 \\
SVI-$k'$           & -89.65 $\pm$ 0.06      & \bf{-89.73 $\pm$ 0.05} & 1.37 $\pm$ 0.15 & 88.36 $\pm$ 0.20 \\
BSVI-$k$-DS        & -88.93 $\pm$ 0.02      & -90.13 $\pm$ 0.04       & 8.13 $\pm$ 0.17 & 81.99 $\pm$ 0.14 \\
BSVI-$k$           & -88.98 $\pm$ 0.03      & -90.19 $\pm$ 0.06       & 8.29 $\pm$ 0.25 & 81.89 $\pm$ 0.20 \\
BSVI-$k$-$\pi$     & -88.95 $\pm$ 0.02      & -90.18 $\pm$ 0.05       & 8.48 $\pm$ 0.22 & \bf{81.70 $\pm$ 0.18} \\
BSVI-$k$-SIR       & \bf{-88.80 $\pm$ 0.03} & -90.24 $\pm$ 0.06      & 7.52 $\pm$ 0.21 & 82.72 $\pm$ 0.22 \\
BSVI-$k$-SIR-$\pi$ & -88.84 $\pm$ 0.05      & -90.22 $\pm$ 0.02       & 7.44 $\pm$ 0.04 & 82.78 $\pm$ 0.05 \\

\end{tabular}
\end{center}
\end{table*}

\begin{table*}[!h]
\small
\centering
\caption{Test set performance on the grayscale SVHN dataset.} 
\label{table:svhn_full_test}
\begin{center}
\begin{tabular}{l|c|ccc}
\textbf{Model} & \textbf{Log-likelihood} & \textbf{ELBO*} & \textbf{KL*} & \textbf{Reconstruction*} \\
\hline
VAE                & -2202.90 $\pm$ 14.95     & -2203.01 $\pm$ 14.96     & 0.40 $\pm$ 0.07  & 2202.62 $\pm$ 14.96 \\
IWAE-$k'$          & -2148.67 $\pm$ 10.11     & -2153.69 $\pm$ 10.94     & 2.03 $\pm$ 0.08  & 2151.66 $\pm$ 10.86 \\
SVI-$k'$           & -2074.43 $\pm$ 10.46     & -2079.26 $\pm$ 9.99      & 45.28 $\pm$ 5.01 & 2033.98 $\pm$ 13.38 \\
BSVI-$k$-DS        & \bf{-2054.48 $\pm$ 7.78} & \bf{-2060.21 $\pm$ 7.89} & 48.82 $\pm$ 4.66 & 2011.39 $\pm$ 9.35 \\
BSVI-$k$           & -2054.75 $\pm$ 8.22      & -2061.11 $\pm$ 8.33      & 51.12 $\pm$ 3.80 & \bf{2009.99 $\pm$ 8.52} \\
BSVI-$k$-$\pi$     & -2060.01 $\pm$ 5.00      & -2065.45 $\pm$ 5.88      & 47.24 $\pm$ 4.62 & 2018.21 $\pm$ 1.64 \\
BSVI-$k$-SIR       & -2059.62 $\pm$ 3.54      & -2066.12 $\pm$ 3.63      & 51.24 $\pm$ 5.03 & 2014.88 $\pm$ 5.30 \\
BSVI-$k$-SIR-$\pi$ & -2057.53 $\pm$ 4.91      & -2063.45 $\pm$ 4.34      & 49.14 $\pm$ 5.62 & 2014.31 $\pm$ 8.25 \\
\end{tabular}
\end{center}
\end{table*}

\begin{table*}[!h]
\small
\centering
\caption{Test set performance on the FashionMNIST dataset.}
\label{table:fmnist_full_test}
\begin{center}
\begin{tabular}{l|c|ccc}
\bf{Model} & \bf{Log-likelihood} & \bf{ELBO*} & \bf{KL*} & \bf{Reconstruction*} \\
\hline
VAE & -1733.86 $\pm$ 0.84      & -1736.49 $\pm$ 0.73      & 11.62 $\pm$ 1.01 & 1724.87 $\pm$ 1.70 \\
IWAE-$k'$ & -1705.28 $\pm$ 0.66      & -1710.11 $\pm$ 0.72      & 33.04 $\pm$ 0.36 & 1677.08 $\pm$ 0.70 \\
SVI-$k'$ & -1710.15 $\pm$ 2.51      & -1718.39 $\pm$ 2.13      & 26.05 $\pm$ 1.90 & 1692.34 $\pm$ 4.03 \\
BSVI-$k$-DS & -1699.14 $\pm$ 0.18      & -1706.92 $\pm$ 0.11      & 41.73 $\pm$ 0.18 & 1665.19 $\pm$ 0.26 \\
BSVI-$k$ & \bf{-1699.01 $\pm$ 0.33} & \bf{-1706.62 $\pm$ 0.35} & 41.48 $\pm$ 0.16 & \bf{1665.14 $\pm$ 0.39} \\
BSVI-$k$-$\pi$ & -1699.24 $\pm$ 0.36      & -1706.92 $\pm$ 0.37      & 41.60 $\pm$ 0.49 & 1665.32 $\pm$ 0.31 \\
BSVI-$k$-SIR & -1699.44 $\pm$ 0.45      & -1707.00 $\pm$ 0.49      & 41.48 $\pm$ 0.12 & 1665.52 $\pm$ 0.41 \\
BSVI-$k$-SIR-$\pi$ & -1699.09 $\pm$ 0.28      & -1706.68 $\pm$ 0.26      & 41.18 $\pm$ 0.19 & 1665.50 $\pm$ 0.31 \\
\end{tabular}
\end{center}
\end{table*}

\begin{table*}[!h]
\small
\centering
\caption{Training set performance on the Omniglot dataset. Note that $k = 9$ and $k' = 10$ (see \cref{sec:setup}). We approximate the log-likelihood with BSVI-$500$ bound (\cref{app:likelihood}). We additionally report the SVI-$500$ bound (denoted ELBO*) along with its KL and reconstruction decomposition.}
\label{table:omni_train}
\begin{center}
\begin{tabular}{l|c|ccc}
\textbf{Model} & \textbf{Log-likelihood} & \textbf{ELBO*} & \textbf{KL*} & \textbf{Reconstruction*} \\
\hline
VAE                & -88.60 $\pm$ 0.18 & -88.66 $\pm$ 0.18 & 1.00 $\pm$ 0.13 & 87.66 $\pm$ 0.19 \\
IWAE-$k'$          & -87.09 $\pm$ 0.12 & \textbf{-87.88 $\pm$ 0.12} & 4.18 $\pm$ 0.19 & 83.70 $\pm$ 0.29 \\
SVI-$k'$           & -88.09 $\pm$ 0.16 & -88.18 $\pm$ 0.15 & 1.38 $\pm$ 0.14 & 86.80 $\pm$ 0.27 \\
BSVI-$k$-SIR       & -87.24 $\pm$ 0.22 & -88.57 $\pm$ 0.25 & 7.67 $\pm$ 0.22 & 80.89 $\pm$ 0.44 \\
BSVI-$k$-DS        & \textbf{-87.00 $\pm$ 0.11}& -88.13 $\pm$ 0.10 & 8.30 $\pm$ 0.18 & 79.83 $\pm$ 0.23 \\
BSVI-$k$           & -87.11 $\pm$ 0.11 & -88.23 $\pm$ 0.10 & 8.45 $\pm$ 0.22 & 79.77 $\pm$ 0.28 \\
BSVI-$k$-$\pi$     & -87.10 $\pm$ 0.11 & -88.24 $\pm$ 0.10 & {8.67 $\pm$ 0.27} & \textbf{79.57 $\pm$ 0.31} \\
BSVI-$k$-SIR-$\pi$ & -87.17 $\pm$ 0.10 & -88.45 $\pm$ 0.11 & 7.63 $\pm$ 0.04 & 80.83 $\pm$ 0.13 \\
\end{tabular}
\end{center}
\end{table*}

\begin{table*}[!h]
\small
\centering
\caption{Training set performance on the grayscale SVHN dataset.} 
\label{table:svhn_train}
\begin{center}
\begin{tabular}{l|c|ccc}
\textbf{Model} & \textbf{Log-likelihood} & \textbf{ELBO*} & \textbf{KL*} & \textbf{Reconstruction*} \\
\hline
VAE                & -2384 $\pm$ 13.58 & -2384 $\pm$ 13.59 & 0.5 $\pm$ 0.09 & 2384 $\pm$ 13.58 \\
IWAE-$k'$          & -2345 $\pm$ 8.77  & -2350 $\pm$ 9.58 & 2.19 $\pm$ 0.03 & 2348 $\pm$ 9.55 \\
SVI-$k'$           & -2274 $\pm$ 8.87  & -2280 $\pm$ 8.34 & 56 $\pm$ 5.75 & 2224 $\pm$ 12.20 \\
BSVI-$k$-SIR       & -2260 $\pm$ 2.73    & -2268 $\pm$ 3.01 & 62.17 $\pm$ 5.51 & 2206 $\pm$ 4.86 \\
BSVI-$k$-DS        & \textbf{-2255.28 $\pm$ 7.38} & \textbf{-2262 $\pm$ 7.51} & 59 $\pm$ 5.34 & 2203 $\pm$ 9.28 \\
BSVI-$k$           & -2255.47 $\pm$ 7.31 & -2263 $\pm$ 7.47 & {62.20 $\pm$ 4.27} & \textbf{2201 $\pm$ 8.26} \\
BSVI-$k$-$\pi$     & -2261 $\pm$ 5.09    & -2268 $\pm$ 6.13 & 58 $\pm$ 5.10 & 2210 $\pm$ 1.43 \\
BSVI-$k$-SIR-$\pi$ & -2258 $\pm$ 3.89    & -2265 $\pm$ 3.30 & 60 $\pm$ 6.36 & 2206 $\pm$ 7.79 \\
% VAE                & -2202.90 $\pm$ 14.95     & -2203.01 $\pm$ 14.96     & 0.40 $\pm$ 0.07  & 2202.62 $\pm$ 14.96 \\
% IWAE-$k'$          & -2148.67 $\pm$ 10.11     & -2153.69 $\pm$ 10.94     & 2.03 $\pm$ 0.08  & 2151.66 $\pm$ 10.86 \\
% SVI-$k'$           & -2074.43 $\pm$ 10.46     & -2079.26 $\pm$ 9.99      & 45.28 $\pm$ 5.01 & 2033.98 $\pm$ 13.38 \\
% BSVI-$k$-SIR       & -2059.62 $\pm$ 3.54      & -2066.12 $\pm$ 3.63      & 51.24 $\pm$ 5.03 & 2014.88 $\pm$ 5.30 \\
% BSVI-$k$-DS        & \bf{-2054.48 $\pm$ 7.78} & \bf{-2060.21 $\pm$ 7.89} & 48.82 $\pm$ 4.66 & 2011.39 $\pm$ 9.35 \\
% BSVI-$k$           & -2054.75 $\pm$ 8.22      & -2061.11 $\pm$ 8.33      & 51.12 $\pm$ 3.80 & \bf{2009.99 $\pm$ 8.52} \\
% BSVI-$k$-$\pi$     & -2060.01 $\pm$ 5.00      & -2065.45 $\pm$ 5.88      & 47.24 $\pm$ 4.62 & 2018.21 $\pm$ 1.64 \\
% BSVI-$k$-SIR-$\pi$ & -2057.53 $\pm$ 4.91      & -2063.45 $\pm$ 4.34      & 49.14 $\pm$ 5.62 & 2014.31 $\pm$ 8.25 \\
\end{tabular}
\end{center}
\end{table*}

\begin{table*}[!h]
\small
\centering
\caption{Training set performance on the FashionMNIST dataset.}
\label{table:fmnist_train}
\begin{center}
\begin{tabular}{l|c|ccc}
\bf{Model} & \bf{Log-likelihood} & \bf{ELBO*} & \bf{KL*} & \bf{Reconstruction*} \\
\hline
VAE                & -1686.11 $\pm$ 2.40 & -1688.84 $\pm$ 2.27 & 11.34 $\pm$ 1.01 & 1677.50 $\pm$ 3.16 \\
IWAE-$k'$          & -1659.12 $\pm$ 0.59 & -1663.80 $\pm$ 0.53 & 33.62 $\pm$ 0.31 & 1630.18 $\pm$ 0.56 \\
SVI-$k'$           & -1666.89 $\pm$ 2.47 & -1675.15 $\pm$ 2.11 & 25.34 $\pm$ 1.80 & 1649.81 $\pm$ 3.90 \\
BSVI-$k$-SIR       & -1653.34 $\pm$ 1.36 & -1660.79 $\pm$ 1.35 & 41.91 $\pm$ 0.15 & 1618.88 $\pm$ 1.51 \\
BSVI-$k$-DS        & -1653.47 $\pm$ 0.85 & -1661.15 $\pm$ 0.82 & {42.13 $\pm$ 0.23} & 1619.02 $\pm$ 1.05 \\
BSVI-$k$           & \textbf{-1652.87 $\pm$ 0.87} & \textbf{-1660.27 $\pm$ 0.89} & 41.85 $\pm$ 0.12 & \textbf{1618.43 $\pm$ 0.86} \\
BSVI-$k$-$\pi$     & -1654.35 $\pm$ 0.74 & -1661.99 $\pm$ 0.68 & 42.00 $\pm$ 0.48 & 1619.99 $\pm$ 1.09 \\
BSVI-$k$-SIR-$\pi$ & -1654.75 $\pm$ 1.19 & -1662.18 $\pm$ 1.25 & 41.58 $\pm$ 0.22 & 1620.60 $\pm$ 1.44 \\
% VAE & -1733.86 $\pm$ 0.84      & -1736.49 $\pm$ 0.73      & 11.62 $\pm$ 1.01 & 1724.87 $\pm$ 1.70 \\
% IWAE-$k'$ & -1705.28 $\pm$ 0.66      & -1710.11 $\pm$ 0.72      & 33.04 $\pm$ 0.36 & 1677.08 $\pm$ 0.70 \\
% SVI-$k'$ & -1710.15 $\pm$ 2.51      & -1718.39 $\pm$ 2.13      & 26.05 $\pm$ 1.90 & 1692.34 $\pm$ 4.03 \\
% BSVI-$k$-SIR & -1699.44 $\pm$ 0.45      & -1707.00 $\pm$ 0.49      & 41.48 $\pm$ 0.12 & 1665.52 $\pm$ 0.41 \\
% BSVI-$k$-DS & -1699.14 $\pm$ 0.18      & -1706.92 $\pm$ 0.11      & 41.73 $\pm$ 0.18 & 1665.19 $\pm$ 0.26 \\
% BSVI-$k$ & \bf{-1699.01 $\pm$ 0.33} & \bf{-1706.62 $\pm$ 0.35} & 41.48 $\pm$ 0.16 & \bf{1665.14 $\pm$ 0.39} \\
% BSVI-$k$-$\pi$ & -1699.24 $\pm$ 0.36      & -1706.92 $\pm$ 0.37      & 41.60 $\pm$ 0.49 & 1665.32 $\pm$ 0.31 \\
% BSVI-$k$-SIR-$\pi$ & -1699.09 $\pm$ 0.28      & -1706.68 $\pm$ 0.26      & 41.18 $\pm$ 0.19 & 1665.50 $\pm$ 0.31 \\
\end{tabular}
\end{center}
\end{table*}

\newpage
\section{Log-likelihood Estimation Using BSVI and IWAE}\label{app:likelihood}

% \rs{A popular way to approximate the log-likelihood is to use the IWAE-$k$ bound, where $k$ is set to a large number during test time evaluation \cite{burda2015importance,sonderby2016ladder,kingma2016improved}. Here we show blah blah blah that BSVI-$k$ is a better alternative} \jw{Done.}

A popular way to approximate the true log-likelihood is to use the IWAE-$k$ bound with a sufficiently large $k$ during evaluation time \cite{burda2015importance,sonderby2016ladder,kingma2016improved}. Here we compare log-likelihood estimates of BSVI and IWAE in Tables \ref{table:omni_iwae_test} and \ref{table:svhn_iwae_test} and empirically show that BSVI bounds are as tight as IWAE bounds in all of our experiments. This justifies the use of BSVI-500 for estimating log-likelihood in our reports.

\begin{table*}[!h]
\centering
\caption{Log-likelihood estimates using BSVI-$k$ and IWAE-$k$ on the Omniglot test set. The tightest estimate is bolded for each model unless there is a tie. Note that $k$ is fixed to $500$, and for IWAE we use five different numbers of particles: $k, 2k, 3k, 4k, 5k$.}
\label{table:omni_iwae_test}
\begin{center}
\begin{tabular}{lcccccccc}
\textbf{Model} & \textbf{BSVI-$k$} & \textbf{IWAE-$k$} & \textbf{IWAE-$2k$} & \textbf{IWAE-$3k$} & \textbf{IWAE-$4k$} & \textbf{IWAE-$5k$} \\ \hline
VAE                & -89.83          & -89.83 & -89.83 & -89.83 & -89.83 & -89.83 \\
SVI-$k'$           & -89.65          & -89.65 & -89.65 & -89.65 & -89.65 & -89.65 \\
IWAE-$k'$          & \textbf{-89.02} & -89.05 & -89.04 & -89.03 & -89.03 & -89.03 \\
BSVI-$k$-DS        & \textbf{-88.93} & -89.05 & -89.00 & -88.99 & -88.98 & -88.97 \\
BSVI-$k$           & \textbf{-88.98} & -89.10 & -89.06 & -89.04 & -89.03 & -89.02 \\
BSVI-$k$-SIR       & \textbf{-88.80} & -88.92 & -88.88 & -88.86 & -88.85 & -88.84 \\
BSVI-$k$-$\pi$     & \textbf{-88.95} & -89.07 & -89.03 & -89.01 & -89.00 & -88.99 \\
BSVI-$k$-SIR-$\pi$ & \textbf{-88.84} & -88.95 & -88.91 & -88.89 & -88.88 & -88.87 \\
\end{tabular}
\end{center}
\end{table*}

\begin{table*}[!h]
\centering
\caption{Log-likelihood estimates using BSVI-$k$ vs. IWAE-$k$ on the SVHN test set. The tightest estimate is bolded for each model unless there is a tie. Note that $k$ is fixed to $500$, and for IWAE we use five different numbers of particles: $k, 2k, 3k, 4k, 5k$.} 
\label{table:svhn_iwae_test}
\begin{center}
\begin{tabular}{lcccccccc}
\textbf{Model} & \textbf{BSVI-$k$}  & \textbf{IWAE-$k$} & \textbf{IWAE-$2k$} & \textbf{IWAE-$3k$} & \textbf{IWAE-$4k$} & \textbf{IWAE-$5k$} \\ \hline
VAE                & -2203          & -2203 & -2203 & -2203 & -2203 & -2203 \\
SVI-$k'$           & \textbf{-2074} & -2096 & -2095 & -2094 & -2094 & -2093 \\
IWAE-$k'$          & -2149          & -2149 & -2149 & -2149 & -2149 & -2149 \\
BSVI-$k$-DS        & \textbf{-2054} & -2079 & -2078 & -2077 & -2077 & -2077 \\
BSVI-$k$           & \textbf{-2055} & -2081 & -2080 & -2080 & -2079 & -2079 \\
BSVI-$k$-SIR       & \textbf{-2060} & -2087 & -2086 & -2085 & -2085 & -2084 \\
BSVI-$k$-$\pi$     & \textbf{-2060} & -2085 & -2083 & -2083 & -2082 & -2082 \\
BSVI-$k$-SIR-$\pi$ & \textbf{-2058} & -2083 & -2082 & -2081 & -2081 & -2080 \\
\end{tabular}
\end{center}
\end{table*}

\begin{table*}[!h]
\centering
\caption{Log-likelihood estimates using BSVI-$k$ vs. IWAE-$k$ on the FashionMNIST test set. The tightest estimate is bolded for each model unless there is a tie. Note that $k$ is fixed to $500$, and for IWAE we use five different numbers of particles: $k, 2k, 3k, 4k, 5k$.} 
\label{table:fmnist_iwae_test}
\begin{center}
\begin{tabular}{lcccccccc}
\textbf{Model} & \textbf{BSVI-$k$}  & \textbf{IWAE-$k$} & \textbf{IWAE-$2k$} & \textbf{IWAE-$3k$} & \textbf{IWAE-$4k$} & \textbf{IWAE-$5k$} \\ \hline
VAE                & \textbf{-1733.86} & -1737.76 & -1737.49 & -1737.35 & -1737.25 & -1737.18 \\
SVI-$k'$           & \textbf{-1705.28} & -1727.30 & -1726.26 & -1725.72 & -1725.35 & -1725.07 \\
IWAE-$k'$          & \textbf{-1710.15} & -1721.01 & -1720.23 & -1719.80 & -1719.51 & -1719.29 \\
BSVI-$k$-DS        & \textbf{-1699.14} & -1727.55 & -1726.37 & -1725.71 & -1725.25 & -1724.93 \\
BSVI-$k$           & \textbf{-1699.01} & -1727.38 & -1726.19 & -1725.53 & -1725.09 & -1724.75 \\
BSVI-$k$-SIR       & \textbf{-1699.24} & -1727.48 & -1726.28 & -1725.63 & -1725.19 & -1724.86 \\
BSVI-$k$-$\pi$     & \textbf{-1699.44} & -1728.05 & -1726.88 & -1726.23 & -1725.77 & -1725.44 \\
BSVI-$k$-SIR-$\pi$ & \textbf{-1699.09} & -1727.03 & -1725.86 & -1725.20 & -1724.77 & -1724.44 \\
\end{tabular}
\end{center}
\end{table*}

\newpage
\section{Experiment Setup}\label{app:setup}

Here we describe our detailed experiment setup. For both Omniglot and SVHN experiments, we used a ResNet with three hidden layers of size 64 as the encoder and a 12-layer gated PixelCNN with the constant layer size of 32 as the decoder. Network parameters ($\phi,\theta$) were trained with the AMSGrad optimizer \cite{reddi2018adam}. For SVI, we followed the experimental setup of \cite{kim2018semi} and optimized local variational parameters $\lambda_{0:k}$ with SGD with momentum with learning rate 1.0 and momentum 0.5. To stabilize training, we applied gradient clipping to both network parameters and local variational parameters. Each model was trained for 200k steps with early-stopping based on validation loss. The best-performing models on the validation set were then evaluated on the test set. All experiments were performed four times, and we reported the mean and standard deviation of relevant metrics.

\textbf{Omniglot}. We used 2000 randomly-selected training images as the validation set. Each digit was dynamically binarized at training time based on the pixel intensity. We used 32-dimensional latent variable with unit Gaussian prior. Each pixel value was modeled as a Bernoulli random variable where the output of the decoder was interpreted as log probabilities. We also followed the training procedure in \cite{kim2018semi} and annealed the KL term multiplier \cite{sonderby2016ladder,bowman2015generating} from $0.1$ to $1.0$ during the first $5000$ iterations of training.

\textbf{SVHN}. We merged ``train'' and ``extra'' data in the original SVHN dataset to create our training set. We again reserved 2000 randomly-selected images as the validation set. To reuse the network architecture for the Omniglot dataset with minimal modifications, we gray-scaled all images and rescaled the pixel intensities to be in $[0,1]$.  The only differences from Omniglot experiments are: increased latent variable dimensions (64), larger image size ($32 \times 32$), and the use of discretized logistic distribution by \cite{salimans2017pixelcnnpp} with a global scale parameter for each pixel. Similar to \cite{tomczak2017vae}, we lower-bound the scale parameter by a small positive value.

\textbf{FashionMNIST}. Similar to above, we used 2000 randomly-selected training images as the validation set. The network architecture and hyperparameters were identical to those of SVHN dataset, except we used 32-dimensional latent variables and did not employ KL term annealing.

% \jw{justify scale lower bound} \rs{Just cite vampprior} \jw{vampprior paper doesn't seem to mention this, so one would have to look at the code to see that they also did this. Is it okay to cite vampprior as the justification?}\rs{The paper links the code, so it's to cite the paper, and state something along the lines of "...similar to [k], we lower bound...".} \jw{Done.}
% \rs{You should also mention that the SVHN images are rescaled to be in [0, 1]} \jw{Done.}

% \jw{TODO(Jay): Discuss/specify: gradient clipping (both global and SVI-specific), batch size, learning rate} \jw{Done.}

% \rs{Remember to state we're also using the EXTRA set in SVHN, and that we're using the GRAYSCALE version so that we can reuse the Omniglot architecture with minimal modifications. The minimal modification is using discretized logistic distribution with a GLOBAL scale variable (just a single scalar trainable parameter).} \jw{Done.}

%\rs{Note to self: rmbr to state how we estimated log-likelihood}

Below is the list hyperparameters used in our experiments. Since we have two stochastic optimization processes (one for the model and one for SVI), we employed separate gradient clipping norms.

\begin{table}[!h]
\centering
\caption{Hyperparameters used for our experiments.} 
\label{table:hyperparams}
\begin{tabular}{l|c|c|c}
\textbf{Hyperparameter} & \textbf{Omniglot} & \textbf{SVHN} & \textbf{FashionMNIST} \\ \hline
Learning rate & $0.001$ & $0.001$ & $0.001$ \\
SVI learning rate & $1.0$ & $1.0$ & $1.0$ \\
SVI momentum & $0.5$ & $0.5$ & $0.5$ \\
Batch size & $50$ & $50$ & $50$ \\
KL-cost annealing steps & $5000$ & $0$ & $0$ \\
Max gradient norm ($\phi,\theta$) & $5.0$ & $5.0$ & $5.0$ \\
Max gradient norm (SVI) & $1.0$ & $1.0$ & $1.0$ \\ \hline
Latent variable dimension & $32$ & $64$ & $32$ \\
Observation model & Bernoulli & Discretized Logistic & Discretized Logistic \\
Scale parameter lower bound & N/A & $0.001$ & $0.001$ \\
\end{tabular}
\end{table}

\end{document}